\documentclass[11pt]{article}
\usepackage[utf8]{inputenc}

\usepackage[margin=1in]{geometry}
\usepackage[hidelinks]{hyperref}
\usepackage{booktabs}
\usepackage[T1]{fontenc}
\usepackage{etoolbox}
\patchcmd{\maketitle}{\@fnsymbol}{\@alph}{}{} 

\usepackage[
    backend=biber,
    style=numeric-comp,
    sortcites,
    sorting=none,
    giveninits=true,
    natbib,
    hyperref=true,
    maxbibnames=99,
    doi=false,isbn=false,url=false,eprint=false
]{biblatex}

\usepackage{placeins}
\usepackage{tikz}
\usepackage{subcaption}

\usepackage{parskip}

\usepackage{xargs}
\usepackage[colorinlistoftodos,prependcaption,textsize=tiny]{todonotes}
\newcommandx{\note}[2][1=]{\todo[linecolor=blue,backgroundcolor=blue!25,bordercolor=blue,#1]{#2}}

\usepackage{algorithm}
\usepackage{algorithmic}
\usepackage{multirow}

\usepackage{xcolor}
\usepackage{graphicx,amsfonts,amsthm,amssymb}
\usepackage{bm}
\usepackage{amsmath,mathrsfs}
\usepackage{mathtools}
\usepackage{cleveref}
\usepackage{stfloats}

\newtheorem{theorem}{Theorem}

\DeclarePairedDelimiter\abs{\lvert}{\rvert}
\DeclarePairedDelimiter\paren{(}{)}

\DeclarePairedDelimiter\cbrac{\{}{\}}

\usepackage{setspace}
\usepackage[colorinlistoftodos]{todonotes}

\title{ConformalHDC: Uncertainty-Aware Hyperdimensional Computing with Application to Neural Decoding}
\author{%
    Ziyi Liang\thanks{Department of Statistics, UC Irvine, Irvine, CA, USA} \quad
    Hamed Poursiami\thanks{Department of Electrical and Computer Engineering, George Mason University, Fairfax, VA, USA} \quad
    Zhishun Yang \footnotemark[1] \textsuperscript{,}\thanks{Center for the Neurobiology of Learning and Memory, UC Irvine, Irvine, CA, USA}  \\
    Keiland Cooper\footnotemark[3] \textsuperscript{,}\thanks{Department of Neurobiology and Behavior, UC Irvine, Irvine, CA, USA}\quad
    Akhilesh Jaiswal\thanks{Department of Electrical and Computer Engineering, University of Wisconsin Madison, Madison, WI, USA} \quad
    Maryam Parsa\footnotemark[2] \\
    Norbert Fortin\footnotemark[3] \textsuperscript{,}\footnotemark[4] \quad
    Babak Shahbaba\footnotemark[1]
}

\usepackage{enumitem}

\begin{document}

\maketitle

\setlist[itemize]{leftmargin=*}
\setlist[enumerate]{leftmargin=*}

\begin{abstract}

Hyperdimensional Computing (HDC) offers a computationally efficient paradigm for neuromorphic learning. Yet, it lacks rigorous uncertainty quantification, leading to open decision boundaries and, consequently, vulnerability to outliers, adversarial perturbations, and out-of-distribution inputs. To address these limitations, we introduce ConformalHDC, a unified framework that combines the statistical guarantees of conformal prediction with the computational efficiency of HDC. For this framework, we propose two complementary variations. First, the set-valued formulation provides finite-sample, distribution-free coverage guarantees. Using carefully designed conformity scores, it forms enclosed decision boundaries that improve robustness to non-conforming inputs. Second, the point-valued formulation leverages the same conformity scores to produce a single prediction when desired, potentially improving accuracy over traditional HDC by accounting for class interactions. We demonstrate the broad applicability of the proposed framework through evaluations on multiple real-world datasets. In particular, we apply our method to the challenging problem of decoding non-spatial stimulus information from the spiking activity of hippocampal neurons recorded as subjects performed a sequence memory task. Our results show that ConformalHDC not only accurately decodes the stimulus information represented in the neural activity data, but also provides rigorous uncertainty estimates and correctly abstains when presented with data from other behavioral states. Overall, these capabilities position the framework as a reliable, uncertainty-aware foundation for neuromorphic computing.

\end{abstract}

\section{Introduction}

Machine learning aims to build data-driven systems capable of making reliable predictions and decisions. Historically, the field has been shaped by two competing paradigms: symbolic AI~\citep{newell-1980-symbols} that uses logical rules and symbols to represent knowledge, making systems easier to understand but often less adaptable; and connectionist AI~\citep{Rumelhart1986LearningRB} that relies on interconnected networks of artificial neurons (e.g., deep neural networks) to learn from data, offering greater flexibility but with reduced transparency and interpretability. In recent years, \emph{Hyperdimensional Computing} (HDC)~\citep{Kanerva2009HyperdimensionalCA, lulu-2020-HDCreview} emerged as a brain-inspired paradigm, aiming to bridge the gap between connectionist and symbolic frameworks. HDC represents information in a structured, symbolic way while retaining the adaptability of the connectionist approaches.

HDC, also known as vector symbolic architectures (VSA), encodes data as binary or real-valued hypervectors with thousands of dimensions, a style consistent with population coding in the brain. Hypervectors are manipulated with operations that differ in their implementation but achieve a similar function to the corresponding operations in the brain~\citep{ni2022algorithm}: bundling (e.g., element-wise addition) to aggregate information, binding (e.g., XOR or circular convolution) to associate concepts, permutation to form spatial or temporal orderings, and similarity (e.g., Hamming or cosine) to retrieve relevant memories~\citep{heddes2024hyperdimensional}. As a result, HDC naturally supports associative memory and fast similarity search. This biological plausibility and interpretability make HDC particularly appealing to neuroscientists. Furthermore, its direct mapping to bitwise logic enables real-time, energy-efficient implementations prized by hardware designers \citep{kleyko2022vector}.


Despite its advantages, HDC has several limitations. First, it lacks rigorous uncertainty quantification, a challenge shared by many machine learning models. Second, its accuracy is often lower than that of traditional algorithms, partly due to simplified decision boundaries. Finally, HDC does not have a mechanism to abstain from predictions when encountering unseen labels. For instance, in face recognition, humans can identify unfamiliar faces as unknown, whereas HDC, and many other classifiers, tend to make a prediction regardless, committing to irresponsible and unreliable decisions on non-conforming inputs.

To address these issues, we introduce ConformalHDC, a novel and robust framework that efficiently combines Hyperdimensional Computing with conformal inference~\citep{vovk1999machine, romano2020classification}. Our contributions are as follows:
\begin{enumerate}
    \item We develop a principled statistical framework for uncertainty quantification in HDC using conformal inference, ensuring valid coverage in a distribution-free, finite-sample setting.
    \item We design novel conformity scores tailored for HDC that are both powerful and computationally efficient, with potential applicability to other similarity-based classifiers.
    \item Our framework enables decision deferral, allowing the model to abstain from predictions when encountering unseen labels or adversarial inputs. By deferring low-confidence cases to a human-in-the-loop, the framework avoids erroneous decisions in high-stakes scenarios.
\end{enumerate}

\subsection{Illustrative Example}
\Cref{fig:hdc-vs-discount} illustrates the intuition and key contributions of our method using a simple classification example. The left panel shows the decision boundary of a standard HDC classifier based on Euclidean distance, other similarity metrics are discussed in \Cref{app:positive-similarity}. The marks in the clusters represent class prototypes, also known as class hypervectors. New test points are first encoded into hypervectors, then assigned the label of the \textit{most similar} class hypervector. As shown, the HDC decision boundaries consist of linear segments. This simplified geometry often overlooks within-class variance and class interactions, which can lead to reduced predictive accuracy. Moreover, these open boundaries partition the entire space, forcing the model to assign a label to every test point, regardless of its distance from the prototypes. As a result, standard HDC methods cannot support decision abstention, since they lack an intrinsic mechanism to recognize points that are too dissimilar from any known class.

In contrast, the right panel of \Cref{fig:hdc-vs-discount} displays the enclosed decision boundaries by ConformalHDC. These boundaries support two variations of our framework: set-valued ConformalHDC and point-valued ConformalHDC. The primary distinction lies in how they handle regions where boundaries overlap. Set-valued ConformalHDC outputs a prediction set that includes the overlapping classes to reflect uncertainty, whereas point-valued ConformalHDC resolves overlaps into a single class assignment for scenarios in which a point prediction is preferred. Notably, both variations return an empty set for test points falling outside all class boundaries. This indicates potential adversarial attacks or out-of-distribution (OOD) labels, enabling expert intervention to ensure more reliable and explainable neuromorphic learning.

As demonstrated in \Cref{sec:synthetic-exp}, point-valued ConformalHDC can achieve higher accuracy than standard HDC by leveraging decision boundaries that account for class interactions. Meanwhile, set-valued ConformalHDC provides finite-sample coverage guarantees, ensuring the prediction set contains the ground-truth label with a probability of at least $1-\alpha$ for a user-specified significance level $\alpha > 0$.

\begin{figure}[h]
    \centering
    \includegraphics[width=0.87\textwidth]{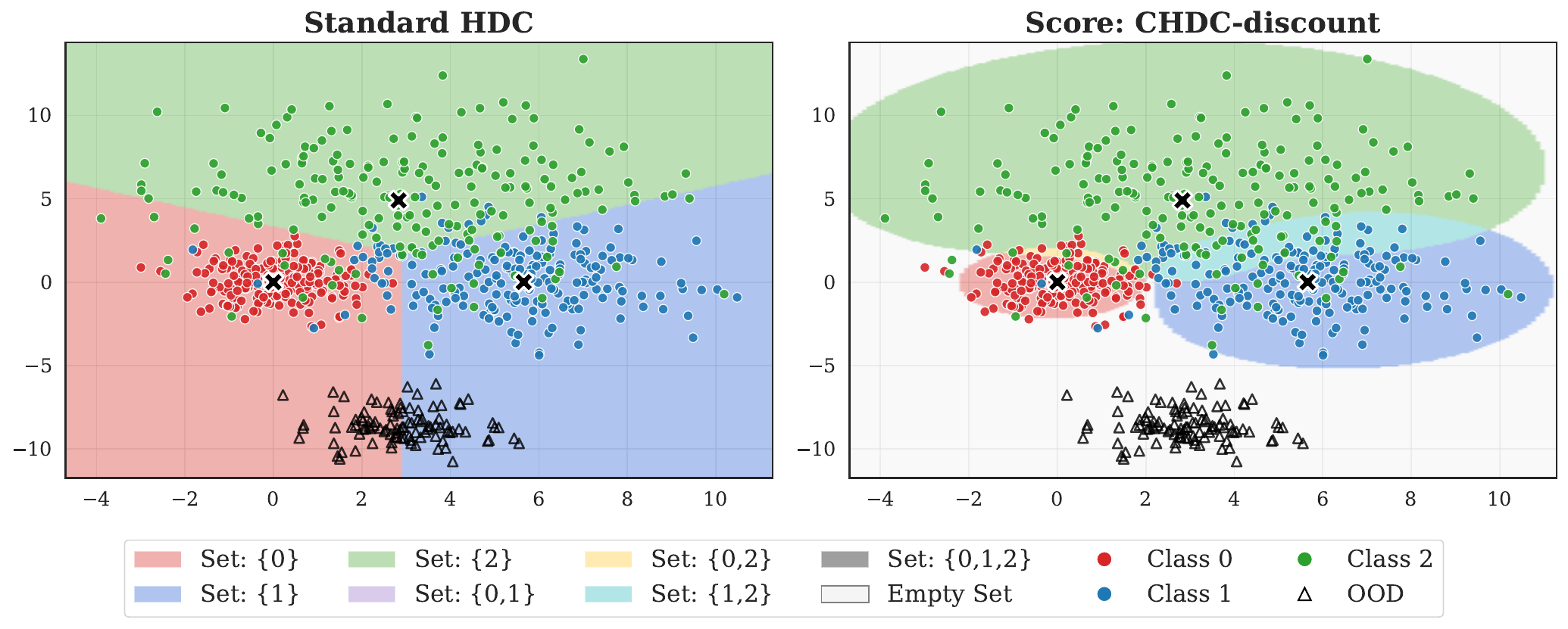} 
     
    \caption{\textbf{ConformalHDC enables rigorous uncertainty quantification and principled abstention through enclosed decision regions.} Unlike the open partitions of standard HDC, our framework produces bounded regions where overlaps explicitly characterize predictive uncertainty. Furthermore, the enclosed geometry allows the model to identify and abstain when encountering non-conforming inputs, such as the OOD samples denoted by triangles. Scatter points represent three classes, and colored regions indicate boundaries.}
    \label{fig:hdc-vs-discount}
\end{figure}

\subsection{Related Work}
Conformal inference~\citep{saunders1999transduction,vovk1999machine,vovk2005algorithmic} has emerged as a prolific research area providing distribution-free uncertainty quantification for machine learning models~\citep{lei2013distribution,lei2014distribution,lei2018distribution,barber2019predictive}. Many prior studies have developed conformal wrappers for pre-trained black-box models across outlier detection~\citep{smith2015conformal,guan2019prediction,Liang_2022_integrative_p_val,bates2021testing}, classification~\citep{vovk2003mondrian,hechtlinger2018cautious,romano2020classification,angelopoulos2020uncertainty,bates2021distributionfree}, and regression~\citep{vovk2005algorithmic,lei2014distribution,romano2019conformalized}.

In contrast to these post-hoc conformal wrappers that treat models as fixed black boxes, our approach integrates conformal inference directly into the internal decision boundaries of HDC models. Although recent works have incorporated conformal guarantees into the training process~\citep{colombo2020training,bellotti2021optimized,stutz2021learning,einbinder2022training}, they primarily target standard machine learning objectives such as model selection or loss regularization, leaving meaningful integration into neuromorphic learning largely unexplored. A recent attempt, HDUQ-HAR~\citep{Lamaakal2025HDUQHAR}, applies conformal prediction within the HDC paradigm but is tailored to human activity recognition and relies on calibration-based hyperparameter tuning to define conformity scores, introducing computational overhead and compromising strict finite-sample coverage guarantees. Our proposed framework, on the other hand, is model-agnostic across diverse modalities (e.g., voice, images, and neural signals), remains hyperparameter-lean with rigorous statistical guarantees, and refines internal decision boundaries rather than operating as a post-hoc wrapper, thereby potentially improving both uncertainty quantification and point-prediction performance.

Building on foundational VSA frameworks~\citep{Kanerva2009HyperdimensionalCA, plate1995holographic, heddes2024hyperdimensional}, methodological research in HDC has focused on two primary directions: developing more expressive and robust encoding schemes to map raw data into high-dimensional space~\citep{yu2022understanding, ni2022algorithm}, and refining the prototyping process. Recent works have introduced adaptive and single-pass training rules that dynamically update class representations to mitigate vector saturation and improve discriminative power~\citep{verges2023refineHD, hernandez2021onlinehd, imani2019adaptHD}. While these advances enhance the quality of learned prototypes, the underlying inference mechanism remains similarity-based classification.

In parallel, application-driven studies have tailored HDC pipelines to meet the constraints of domains such as biosignals classification~\citep{ni2022neurally, yu2024lifelong}, wearable health monitoring~\citep{shahhosseini2022flexible}, autonomous vehicles sensor readings~\citep{wang2021brief},
and structured graph learning~\citep{ nunes2022graphhd, poursiami2025vs, cong2025hypergraphx}. Supporting these deployments, systems-level research has emphasized hardware efficiency through fully binary pipelines and memory-centric designs for energy-constrained platforms~\citep{gupta2020thrifty, imani2019bric}.

Our work builds upon this thriving ecosystem by introducing a versatile uncertainty quantification layer that is fully compatible with these developments. As a model-agnostic framework, ConformalHDC requires only class prototypes, independent of the encoding strategy or training regime, allowing seamless integration into existing HDC pipelines. Moreover, it preserves the memory and computational efficiency necessary to satisfy stringent hardware constraints.


\section{Preliminary}
\subsection{HDC classifiers}\label{sec:review-hdc}
Consider $n$ pairs of observations $(X_i,Y_i)$, for $i \in \mathcal{D} = [n] = \cbrac{1, \dots, n}$, and a test point $(X_{n+1}, Y_{n+1})$ whose label $Y_{n+1} \in [K]$ has not yet been observed. HDC classifiers aim to produce a prediction $\widehat{Y}_{n+1}$ given $X_{n+1}$ and the observations indexed by $\mathcal{D}$. To this end, they leverage the ``blessing of dimensionality," based on the principle that randomly sampled vectors from a high-dimensional space are nearly orthogonal with high probability due to concentration of measure~\citep{Ledoux2001concentration}. Therefore, projecting features into a high-dimensional space can render complex data patterns that remain easily differentiable.

The methodology consists of two key components. First, an encoder, denoted as $\phi: \mathcal{X} \mapsto \mathcal{H}$, maps an input feature vector $X \in \mathcal{X}$ to a high-dimensional hypervector $H = \phi(X)$. This encoding, often achieved using methods such as Compositional Encoding~\citep{rahimi2016compositional, kleyko2014ngram} or Random Projection~\citep{rahimi2007RFF}, is designed to preserve the similarity of the original data points in the high-dimensional space $\mathcal{H}$. Second, for each class $y\in[K]$, a \textit{prototype} hypervector, $H^y \in \mathcal{H}$, is constructed from the corresponding hypervectors of training samples with that specific label~\citep{Rahimi2016ARA, imani2019adaptHD, verges2023refineHD}. During inference, the test feature $X_{n+1}$ is encoded into a query hypervector $H_{n+1} =\phi(X_{n+1})$, and its label is predicted by identifying the class prototype with the highest similarity, namely, $\widehat{Y}_{n+1} = \arg\max_{y \in [K]} \delta(H_{n+1}, H^y)$, where $\delta(\cdot, \cdot)$ can be any similarity metrics, including Cosine Similarity or Hamming Distance~\citep{lulu-2020-HDCreview}.

\subsection{Conformal inference}\label{sec:review-conformal}
Instead of providing a point prediction, the objective of conformal inference is to construct a prediction set $\widehat{C}_{\alpha}(X_{n+1})$ for $Y_{n+1}$, given a fixed significance level $\alpha \in (0,1)$. The ideal goal is to create the smallest possible prediction set that guarantees {\em feature-conditional coverage} at level $1-\alpha$, meaning $\mathbb{P}\cbrac{Y_{n+1} \in \widehat{C}_{\alpha}(X_{n+1}) \mid X_{n+1} = x} \geq 1-\alpha$ for any feature vector $x \in \mathcal{X}$.
However, achieving this is often infeasible unless the feature space $\mathcal{X}$ is very small~\citep{barber2019limits}. Consequently, one often seeks more practical guarantees, such as {\em label-conditional coverage} and {\em marginal coverage}. Formally,  $\widehat{C}_{\alpha}(X_{n+1})$ has $1-\alpha$ label-conditional coverage if $\mathbb{P} \cbrac{Y_{n+1} \in \widehat{C}_{\alpha}(X_{n+1}) \mid Y_{n+1} = y} \geq 1-\alpha$, for any $y \in [K]$, and marginal coverage if $\mathbb{P}\cbrac{Y_{n+1} \in \widehat{C}_{\alpha}(X_{n+1})} \geq 1-\alpha$. While label-conditional coverage is stronger, the marginal criterion is also valuable as it can be easier to achieve with smaller, more informative prediction sets.

Standard split conformal inference partitions $\mathcal{D}$ into a training set, $\mathcal{D}_{\text{train}}$, and a calibration set, $\mathcal{D}_{\text{cal}}$. A black-box classifier (in our case, a HDC classifier) is trained on $\mathcal{D}_{\text{train}}$. The holdout data in $\mathcal{D}_{\text{cal}}$ is then used to calibrate the model via a \textit{nonconformity score} function $\widehat{S}: \mathcal{X}\times[K] \mapsto \mathbb{R}$, where smaller value indicate the feature ``conforms'' with the given class. For each calibration point $i \in \mathcal{D}_{\text{cal}}$, we compute the score $\widehat{S}(X_i, Y_i)$ and find the $1-\alpha$ empirical quantile of these scores, denoted $\widehat{Q}_{1-\alpha}$ (assuming no ties for simplicity). The final prediction set is then constructed as: 
\begin{align}\label{eq:conformal-pset}
    \widehat{C}_{\alpha}(X_{n+1})=\cbrac{y \in [K]: \widehat{S}(X_{n+1}, y) \leq \widehat{Q}_{1-\alpha}}.
\end{align} 
If the data points in $(X_{n+1}, Y_{n+1}) \cup \cbrac{(X_i, Y_i)}_{i\in \mathcal{D}}$ are exchangeable (a weaker requirement than i.i.d.), $\widehat{C}_{\alpha}(X_{n+1})$ offers finite-sample marginal coverage without any distributional assumptions other than exchangeability~\citep{romano2020classification}. Stronger label-conditional guarantees can be obtained by stratifying the calibration set by labels and computing label-wise quantiles as detailed in \Cref{sec:set-valued-CHDC}.

\section{Methodology}\label{sec:method}
\subsection{Nonconformity scores for HDC}\label{sec:conformity-scores}
The procedure reviewed in \Cref{sec:review-conformal} provides a general framework of constructing conformal prediction sets, but the key component, the nonconformity score function, remains implicit. In this section, we argue that existing scores are ill-suited for HDC classifiers and introduce a novel nonconformity score that is both computationally efficient and specifically tailored to the HDC architecture.

Existing nonconformity scores present significant challenges for HDC. Standard scores, such as the inverse quantile score proposed by~\citet{romano2020classification} (reviewed in \Cref{app:class-scores}), assume that for any $x \in \mathcal{X}$ and $y \in [K]$, the classifier outputs softmax probabilities $\widehat{\pi}_{y}(x) \in [0,1]$, e.g., from the softmax output layer of a deep neural network, that approximate the true conditional distribution, $\mathbb{P}(Y=y \mid X=x)$, and $\sum_{y \in [K]}\widehat{\pi}_{y}(x) = 1$. HDC classifiers, however, output raw similarity measures, $\delta(h, H^y)$, not probabilities. Passing these similarities through a softmax function constitutes a flawed workaround, as it normalizes away critical information necessary for prediction deferral. For example, any point along a classical HDC decision boundary is equidistant from two class prototypes, yielding identical normalized probabilities (e.g., [0.5, 0.5]). As a result, these points become indistinguishable, hindering the classifier's ability to abstain. As shown in \Cref{fig:boundary-comparison}, the inverse quantile score leads to decision boundaries that are unable to perform outlier detection and lose meaningful geometric interpretation.

A different line of work constructs nonconformity scores via class-conditional density estimation. For example, \citet{hechtlinger2018cautious} propose to estimate $\mathbb{P}(X=x \mid Y=y)$ using methods like Kernel Density Estimation (KDE). 
Using this idea, we construct a benchmark, called penalized score: 
\begin{equation}\label{eq:score-penalized}
\begin{aligned}
    \widehat{S}_{\mathrm{pen}}(X, Y) &=\widehat{S}_{\mathrm{pen}}(X, Y; \phi, \cbrac{H^y}_{y\in[K]}) \\&=-\delta(\phi(X),H^{Y}) +\lambda \sum\nolimits_{y \neq Y} \delta(\phi(X),H^{y}).
\end{aligned}
\end{equation}
where $\lambda \geq 0$. It mimics the correlated density estimation $\mathbb{P}(X=x \mid Y=y)-\lambda\sum\nolimits_{y'\neq y}\mathbb{P}(X=x \mid Y=y')$ which models class interactions by penalizing high density regions for the other classes \citep{hechtlinger2018cautious}. The primary limitation of this method lies in its strong sensitivity to the hyperparameter $\lambda$. As shown in \Cref{fig:boundary-comparison}, an appropriate $\lambda$ can prevent meaningful abstention, while if $\lambda$ is too small, the score degenerates towards a similarity only score in \Cref{eq:score-sim} which overlooks class interactions.

To address these issues, we propose the following nonconformity score (CHDC-discount) tailored for HDC classifiers:
\begin{equation}\label{eq:score-adj-sim}
\begin{aligned}
    \widehat{S}(X, Y) & =\widehat{S}(X, Y; \phi, \cbrac{H^y}_{y\in[K]}) \\ &= -\frac{\delta(\phi(X),H^{Y})}{\sum_{y \in [K]} \delta(\phi(X),H^{y})}\cdot \delta(\phi(X),H^{Y}).
\end{aligned}
\end{equation} 
This score can be understood as the base similarity to the correct class prototype, $\delta(\phi(x),H^{Y})$, weighted by a ratio that approximates the likelihood of the correct class versus all others. Note that by convention, smaller nonconformity score indicate stronger evidence of $X$ belonging to class $Y$. 

Intuitively, our proposed score in \eqref{eq:score-adj-sim} is constructed from two key components: a similarity term that enables abstention, and a ratio term that models class interactions. To analyze their individual contributions, we define two alternative scores based on each component in isolation: the similarity score,
\begin{align}\label{eq:score-sim}
    \widehat{S}_{\mathrm{sim}}(X, Y) = -\delta(\phi(X),H^{Y}),
\end{align}
and the ratio score (CHDC-ratio),
\begin{align}\label{eq:score-ratio}
    \widehat{S}_{\mathrm{ratio}}(X, Y) = -\frac{\delta(\phi(X),H^{Y})}{\sum_{y \in [K]}\delta(\phi(X),H^{y})}.
\end{align}
To ensure the scores in \eqref{eq:score-penalized}--\eqref{eq:score-ratio} are well-defined, we require that $\delta \geq 0$. This is a standard practice in the literature and is achieved by normalizing the similarity to a nonnegative range, such as $[0, 1]$; see \Cref{app:positive-similarity} for further details.

As illustrated in \Cref{fig:boundary-comparison}, using the similarity score, $\widehat{S}_{\mathrm{sim}}$, results in significant overlap between class boundaries. This produces less informative prediction sets, as points in these ambiguous regions are assigned multiple labels. In contrast, the ratio score, $\widehat{S}_{\mathrm{ratio}}$, produces sharper boundaries by explicitly modeling class interactions, but it loses the ability to abstain (\Cref{fig:boundary-comparison}). Because it captures only relative confidence, it cannot distinguish an outlier point far from all prototypes from an inlier point with the same ratio score.

In contrast, the discount score in \eqref{eq:score-adj-sim} balances separation and abstention effectively without introducing a sensitive hyperparameter that requires additional tuning. 

\begin{figure*}[!t]
    \centering 
    \includegraphics[width=\textwidth]{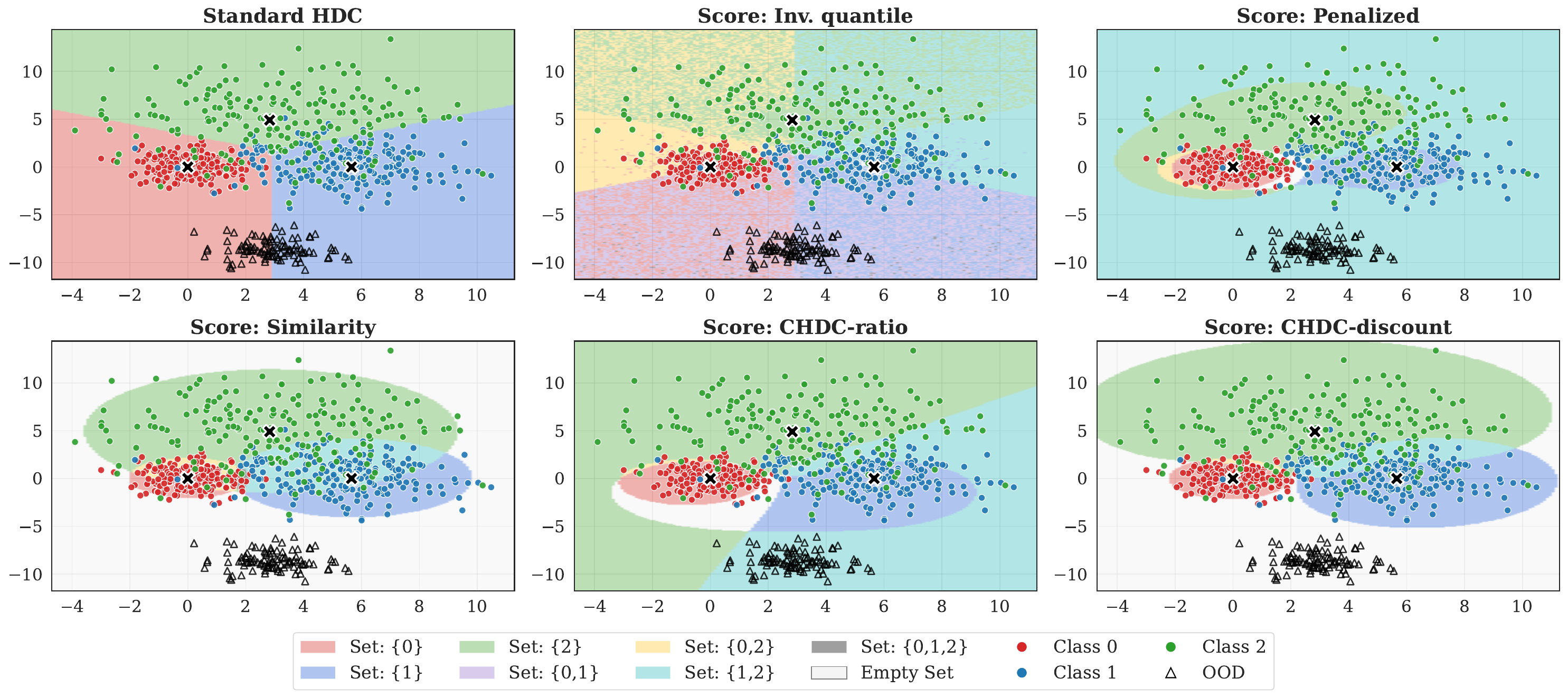}
    \caption{Decision boundary comparison between standard HDC and conformal methods with different nonconformity scores in \cref{sec:conformity-scores}. Other details remain the same as in \Cref{fig:hdc-vs-discount}.}
    \label{fig:boundary-comparison}
\end{figure*}

\subsection{Set-valued ConformalHDC}\label{sec:set-valued-CHDC}
With the nonconformity scores defined, we are ready to state the set-valued conformalHDC which outputs a prediction set $\widehat{C}_{\alpha}(X_{n+1})$.
Given a nonconformity score, \Cref{alg:set-conformalHDC-marginal} formalizes the procedure for generating conformal prediction sets. Subsequently, \Cref{thm:marginal} establishes that this approach produces prediction sets satisfying a tight marginal coverage guarantee.
\begin{theorem}\label{thm:marginal}
Assume $(X_{1},Y_{1}), \ldots, (X_{n+1},Y_{n+1})$ are exchangeable, and let $\widehat{C}_{\alpha}(X_{n+1})$ be the output of Algorithm~\ref{alg:set-conformalHDC-marginal}. Then, for any $\alpha \in (0,1)$,
\begin{align*}
    \mathbb{P}\cbrac{Y_{n+1} \in \widehat{C}_{\alpha}(X_{n+1})} \geq 1-\alpha.
\end{align*}
Furthermore, if the scores $\widehat{S}_i$ are almost surely distinct, the coverage is upper bounded by
\begin{align*}
    \mathbb{P}\cbrac{Y_{n+1} \in \widehat{C}_{\alpha}(X_{n+1})} \leq 1-\alpha + \frac{1}{1+|\mathcal{D}_{\mathrm{cal}}|}.
\end{align*}
The probabilities above are taken over the randomness of $(X_{1},Y_{1}), \ldots, (X_{n+1},Y_{n+1})$.
\end{theorem}

\begin{algorithm}[h]
    \caption{Set-valued CHDC with marginal coverage}
    \label{alg:set-conformalHDC-marginal}
    \begin{algorithmic}[1]
        \STATE \textbf{Input}: Exchangeable data $(X_{1},Y_{1}), \ldots, (X_{n},Y_{n})$ with labels $Y_i \in [K]$; test features $X_{n+1}$; significance level $\alpha$; a prototype-based HDC classifier.
        \STATE Randomly split the data points into $\mathcal{D}_{\text{train}}$ and $\mathcal{D}_{\text{cal}}$.
        \STATE Train the HDC classifier on $\mathcal{D}_{\text{train}}$ to obtain the encoder $\phi$, and class prototypes $\cbrac{H^y}_{y\in[K]}$.
        \STATE Compute scores $\widehat{S}_i=\widehat{S}_i(X_{i},Y_i)$, $\forall i \in \mathcal{D}_{\text{cal}}$, using \eqref{eq:score-adj-sim}.
        \STATE Compute $\widehat{Q}_{1-\alpha} \coloneqq \widehat{Q}_{1-\alpha}(\{\widehat{S}_i\}_{i \in \mathcal{D}_{\text{cal}}})$ as the $\lceil (1-\alpha)(1+|\mathcal{D}_{\text{cal}}|) \rceil$-th smallest value among $\{\widehat{S}_i\}_{i \in \mathcal{D}_{\text{cal}}}$.
        \STATE Compute the prediction set for the unobserved label $Y_{n+1}$, given $X_{n+1}$, by \eqref{eq:conformal-pset}.
        \STATE \textbf{Output}:  Prediction set $\widehat{C}_{\alpha}(X_{n+1})$ with marginal coverage guarantee.
    \end{algorithmic}
\end{algorithm}

In certain settings, stronger label-conditional coverage is required, particularly in contexts where fairness is a concern or the result is label-sensitive. To this end, we derive a variation of the prediction sets that produces such label-conditional coverage by computing label dependent empirical quantiles. \Cref{alg:set-conformalHDC-cond} summarizes the procedure and \Cref{thm:conditional} below proves its validity. 
\begin{theorem}\label{thm:conditional}
Assume $(X_{1},Y_{1}), \ldots, (X_{n+1},Y_{n+1})$ are exchangeable, and let $\hat{C}_{\alpha}(X_{n+1})$ be the output of Algorithm~\ref{alg:set-conformalHDC-cond}. Then, for any $\alpha \in (0,1)$ and $y \in [K]$,
\begin{align*}
    \mathbb{P}\cbrac{Y_{n+1} \in \hat{C}_{\alpha}(X_{n+1}) \mid Y_{n+1}=y} \geq 1-\alpha.
\end{align*}
Furthermore, if the scores $\hat{S}_i$ are almost surely distinct, the coverage is upper bounded by
\begin{align*}
    \mathbb{P}\cbrac{Y_{n+1} \in \hat{C}_{\alpha}(X_{n+1})\mid Y_{n+1}=y} \leq 1-\alpha + \frac{1}{1+|\mathcal{D}^y_{\mathrm{cal}}|}, 
\end{align*}
where $\mathcal{D}^y_{\mathrm{cal}} = \{i \in \mathcal{D}_{\mathrm{cal}} : Y_i = y \}$. The probabilities above are taken over the randomness of $(X_{1},Y_{1}), \ldots, (X_{n+1},Y_{n+1})$.
\end{theorem}

All proofs are provided in \Cref{app:proofs}.

\begin{algorithm}[h]
    \caption{Set-valued ConformalHDC with label-conditional coverage}
    \label{alg:set-conformalHDC-cond}
    \begin{algorithmic}[1]
       \STATE \textbf{Input}: Same as \Cref{alg:set-conformalHDC-marginal}.
        \STATE Randomly split the data points into $\mathcal{D}_{\text{train}}$ and $\mathcal{D}_{\text{cal}}$.
        \STATE Train the HDC classifier on $\mathcal{D}_{\text{train}}$ to obtain the encoder $\phi$, and class prototypes $\cbrac{H^y}_{y\in[K]}$.
        \STATE Compute scores $\hat{S}_i=\hat{S}_i(X_{i},Y_i)$, $\forall i \in \mathcal{D}_{\text{cal}}$, using \eqref{eq:score-adj-sim}.
        \FOR{$ y \in [K]$}
        \STATE Define $\mathcal{D}^y_{\text{cal}} = \{i \in \mathcal{D}_{\text{cal}} : Y_i = y \}$.
        \STATE Compute $\widehat{Q}^y_{1-\alpha} \coloneqq \widehat{Q}^y_{1-\alpha}(\{\hat{S}_i\}_{i \in \mathcal{D}^y_{\text{cal}}})$ as the $\lceil (1-\alpha)(1+|\mathcal{D}^y_{\text{cal}}|) \rceil$-th smallest value among $\{\hat{S}_i\}_{i \in \mathcal{D}^y_{\text{cal}}}$.
        \ENDFOR
        \STATE Compute the prediction set by \begin{equation}\label{eq:conditional-pset}
            \hat{C}_{\alpha}(X_{n+1})=\cbrac{y \in [K]: \hat{S}(X_{n+1}, y) \leq \widehat{Q}^y_{1-\alpha}}.
        \end{equation}
        \vspace{-1em}
        \STATE \textbf{Output}:  Prediction set $\hat{C}_{\alpha}(X_{n+1})$ with label-conditional coverage guarantee.
    \end{algorithmic}
\end{algorithm}

\subsection{Point-valued ConformalHDC}\label{sec:point-valued-CHDC}
Although the set-valued conformalHDC procedures in \Cref{alg:set-conformalHDC-marginal} and \Cref{alg:set-conformalHDC-cond} offer rigorous coverage guarantees, in practice point-valued predictors are sometimes preferred since most downstream decision-making depends on a single prediction rather than a prediction set. To this end, \Cref{alg:point-conformalHDC} presents a point-valued ConformalHDC procedure that trims the prediction set. More specifically, it refines the prediction set by \Cref{alg:set-conformalHDC-marginal} or \Cref{alg:set-conformalHDC-cond} whenever $\widehat{C}_{\alpha}(X_{n+1})$ contains more than one class, retaining only the label with the smallest nonconformity score as defined in \Cref{eq:score-adj-sim}. Consequently, the point-valued conformalHDC by \Cref{alg:point-conformalHDC} always outputs either a singleton or the empty set.

Compared to standard HDC, which is also a point-valued predictor, point-valued conformalHDC leverages the more informative decision boundaries visualized in \Cref{fig:hdc-vs-discount}. As observed in \Cref{sec:synthetic-exp}, this can lead to improved classification performance by accounting for class interactions. Furthermore, when facing distribution shifts or data contamination, it retains the ability to abstain. By outputting an empty set, it identifies spurious test points that should be deferred to a human-in-the-loop or a more sophisticated expert model. For scenarios where practitioners are confident that test points come from inlier classes only, \Cref{alg:point-conformalHDC} provides the flexibility to disable this feature, defaulting to the most similar class instead of an empty set.

\begin{algorithm}[h]
    \caption{Point-valued CHDC}
    \label{alg:point-conformalHDC}
    \begin{algorithmic}[1]
        \STATE \textbf{Input}: Boolean flag \texttt{allow\_empty} indicating if empty set is allowed.
        \STATE Compute $\widehat{C}_{\alpha}(X_{n+1})$ by \Cref{alg:set-conformalHDC-marginal} or \Cref{alg:set-conformalHDC-cond}.
        \IF{$\abs{\widehat{C}_{\alpha}(X_{n+1})}>1$}
        \STATE Update $\widehat{C}_{\alpha}(X_{n+1})$ with the singleton 
         \vspace{-0.2em}
        \begin{align*}
            \widehat{C}_{\alpha}(X_{n+1}) = \cbrac*{\arg\min\nolimits_{y \in \widehat{C}_{\alpha}(X_{n+1})} \widehat{S}(X_{n+1},y)}.
        \end{align*}
        \ELSIF{$\abs{\widehat{C}_{\alpha}(X_{n+1})}=0$ and not \texttt{allow\_empty}}
        \STATE \textcolor{white}{s} 
        \vspace{-1em}
        \begin{align*}
            \widehat{C}_{\alpha}(X_{n+1}) = \cbrac*{\arg\min\nolimits_{y \in [K]} \widehat{S}(X_{n+1},y)}.
        \end{align*}
        \vspace{-1em}
        \ENDIF
        \STATE \textbf{Output}:  Trimmed prediction set $\widehat{C}_{\alpha}(X_{n+1})$.
    \end{algorithmic}
\end{algorithm}

\subsection{Computational Complexity}
The proposed ConformalHDC framework is designed for high computational efficiency and scalability, maintaining the high-speed execution characteristic of prototype-based architectures. For all proposed algorithms, the total computational complexity scales as $\mathcal{O}(T + n_{\text{cal}}Kd + mKd)$, where $T$ denotes the base training cost (including encoding and prototype construction), $n_{\text{cal}}=\abs{\mathcal{D}_{\text{cal}}}$ is the calibration size, $d$ is hypervector dimension, and $m$ is the number of test points. Since $n_{\text{cal}}$ is typically much smaller than the training set size and the inference cost remains linear with respect to dimension $d$ and class count $K$, the framework offers a rigorous yet lightweight solution for uncertainty quantification in resource-constrained environments. A detailed complexity breakdown is provided in \Cref{app:complexity}.

\section{Synthetic Experiment}\label{sec:synthetic-exp}
To empirically validate the ConformalHDC framework, we simulate a classification task using a Gaussian mixture model that mimics HDC behavior in feature space. In this setup, class prototypes serve as centers of localized data clusters, and the inverse Euclidean distance is used as the similarity metric for score computation. To reflect real-world scenarios where classes exhibit varying degrees of intra-class variability, we fix the dispersion of two classes and vary the standard deviation $\sigma$ of the third from $3$ to $5$ by increment of $0.25$. 
This control parameter allows us to evaluate the framework's ability to adaptively expand prediction sets for high-variance classes while maintaining precise sets for more stable ones. 

\begin{figure}[h]
    \centering
    \includegraphics[width=0.9\linewidth]{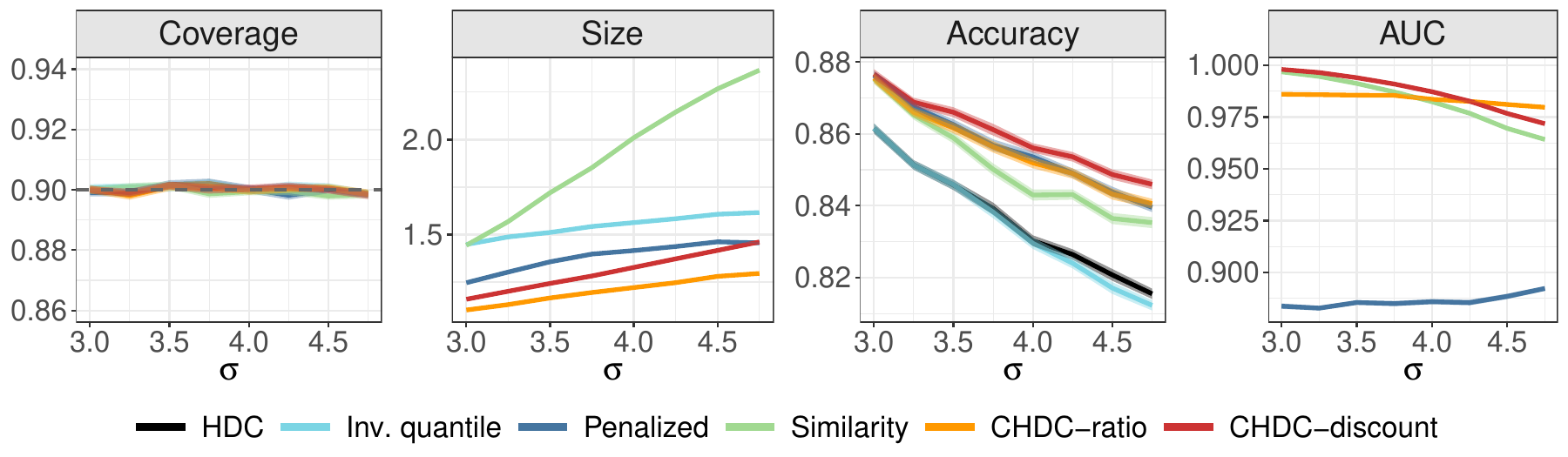}
    \caption{ConformalHDC outperforms benchmarks by providing smaller prediction sets, enhanced point-prediction accuracy, and superior OOD detection. Data generation process is detailed in \Cref{app:synthetic-3class} with $\sigma$ controlling the heteroscedasticity of the data. Nominal coverage level is $90\%$. Results are summarized over 100 repetitions.}
    \label{fig:3class}
\end{figure}


We evaluate the performance of standard HDC against conformal-based methods using different non-conformity scores. We assess the marginal coverage and set size for set-valued outputs. As shown in \Cref{fig:3class}, while all conformal-based methods maintain valid coverage, set-valued ConformalHDC employing ratio and discount scores (CHDC-ratio and CHDC-discount) consistently produces smaller, and thus more informative, prediction sets. Compared to standard HDC and other benchmarks, point-valued ConformalHDC achieves higher prediction accuracy by effectively accounting for class interactions. For OOD detection, the proposed discount score demonstrates robust performance, maintaining a high AUC (area under the curve) across a range of class dispersions. We omit the penalized score from this plot, as its AUC is approximately 0.5 (see \Cref{tab:3-class} in \Cref{app:additional-results}).

\section{Real-data Applications}\label{sec:real-exp}
\subsection{Application to benchmark datasets}\label{sec:exp-benchmarks}

To evaluate ConformalHDC across multiple data modalities, we consider three standard benchmark datasets from the image, speech, and text domains. Each dataset represents a distinct classification task:
\begin{itemize}
    \item \textbf{Image (MNIST).} The MNIST database  \citep{deng2012mnist, manabat2019performance} contains 70,000 grayscale images of handwritten digits (0–9), each with a resolution of 28×28.
    \item \textbf{Voice (ISOLET).} The ISOLET dataset \citep{isolet_54, imani2017voicehd} consists of approximately 7,800 spoken recordings of the 26 English alphabet letters, represented by fixed-length acoustic feature vectors.
    \item \textbf{Text (European Languages).} The European Languages dataset \citep{quasthoff2006corpus, joshi2016language} contains around 21,000 text samples from 21 European languages.
\end{itemize}
For each dataset, several classes are reserved as OOD samples to evaluate abstention performance. Detailed implementation, including the encoding processes, are provided in \Cref{app:benchmark-encoding}. \Cref{tab:real-data} summarizes the marginal coverage and average prediction set size for the uncertainty quantification task, alongside the AUC for OOD detection. While all conformal methods maintain valid coverage, CHDC-ratio, CHDC-discount, and inverse quantile scores generally yield smaller prediction sets. Notably, CHDC-discount consistently excels in OOD detection with high AUC values, whereas the inverse quantile score, as expected, fails in this setting. Comprehensive accuracy comparisons for point-valued predictors are provided in \Cref{tab:mnist}--\Cref{tab:isolet}  in \Cref{app:additional-results}.

\begin{table}[h]
    \centering
    \small                         
    \renewcommand{\arraystretch}{0.9} 
    \setlength{\tabcolsep}{6.8pt}
    \resizebox{\textwidth}{!}{
\begin{tabular}{|l|ccc|ccc|ccc|ccc|}
   \hline
 & \multicolumn{3}{c|}{\textbf{MNIST ($\alpha=0.05$)}} & \multicolumn{3}{c|}{\textbf{Languages ($\alpha=0.01$)}} & \multicolumn{3}{c|}{\textbf{ISOLET ($\alpha=0.02$)}} & \multicolumn{3}{c|}{\textbf{Neural ($\alpha=0.2$)}} \\
\cline{2-13}
\textbf{Method} & \textbf{Cov.} & \textbf{Size} & \textbf{AUC} & \textbf{Cov.} & \textbf{Size} & \textbf{AUC} & \textbf{Cov.} & \textbf{Size} & \textbf{AUC} & \textbf{Cov.} & \textbf{Size} & \textbf{AUC} \\
\hline
HDC & 0.862 & 1.000 & - & 0.953 & 1.000 & - & 0.889 & 1.000 & - & 0.640 & 1.000 & - \\ 
  Inv. quantile & 0.950 & 2.002 & 0.483 & 0.990 & 2.514 & 0.472 & 0.980 & 3.415 & 0.490 & 0.799 & 1.782 & 0.599 \\ 
  Penalized & 0.950 & 4.218 & 0.655 & 0.990 & 16.392 & 0.294 & 0.980 & 21.344 & 0.536 & 0.796 & 2.901 & 0.016 \\ 
  Similarity & 0.951 & 3.370 & 0.745 & 0.990 & 9.712 & 0.960 & 0.980 & 7.428 & 0.759 & 0.792 & 2.979 & 0.986 \\ 
  CHDC-ratio & 0.950 & 1.770 & 0.873 & 0.990 & 3.291 & 0.982 & 0.979 & 4.412 & 0.759 & 0.796 & 1.643 & 0.600 \\ 
  CHDC-discount & 0.950 & 2.425 & 0.815 & 0.990 & 5.354 & 0.973 & 0.980 & 5.401 & 0.778 & 0.791 & 2.761 & 0.986 \\ 
   \hline
\end{tabular}
}
    \caption{Performance comparison of ConformalHDC and benchmarks under real data applications. Results for benchmark datasets in \Cref{sec:exp-benchmarks} are averaged over 100 repetitions. See \Cref{tab:mnist}--\Cref{tab:isolet} for full results including standard errors. Results for the neural decoding experiment in \cref{sec:exp-neuron} are averaged over 500 repetitions. While results for one representative subject (Buchanan) are shown here, complete results for all five subjects are provided in \Cref{tab:rat}.}
    \label{tab:real-data}
\end{table}

\subsection{Application to Neural Decoding}\label{sec:exp-neuron}

We also apply our framework to a neural dataset involving hippocampal activity recorded from rats performing a complex sequence memory task~\citep{shahbaba2022hippocampal}. Our objectives are threefold: first, we encode temporal trajectories of population activity into a high-dimensional representation; second, we provide decoding (prediction) with rigorous uncertainty quantification; and third, we differentiate the activity between distinct behavioral states (odor sampling vs running).

In this task, well-trained rats received multiple presentations of a sequence of 5 odor stimuli (e.g., ABCDE) at a single odor port. For each odor, the rats had to correctly determine if the odor was presented in the expected order (by holding their nose in the port until the signal at 1.2s) to earn a water reward. To ensure clear segmentation between sequences, rats had to run to the opposite end of the track and return to the port before the next sequence could be initiated. This running period was used to define a "Run" state (0.2s period at the center of the track) that served as OOD data, as neural activity is known to differ between odor sampling and running behavior. \citep{allen2016nonspatial}

\begin{figure}[!h]
    \centering
    \includegraphics[width=1\linewidth]{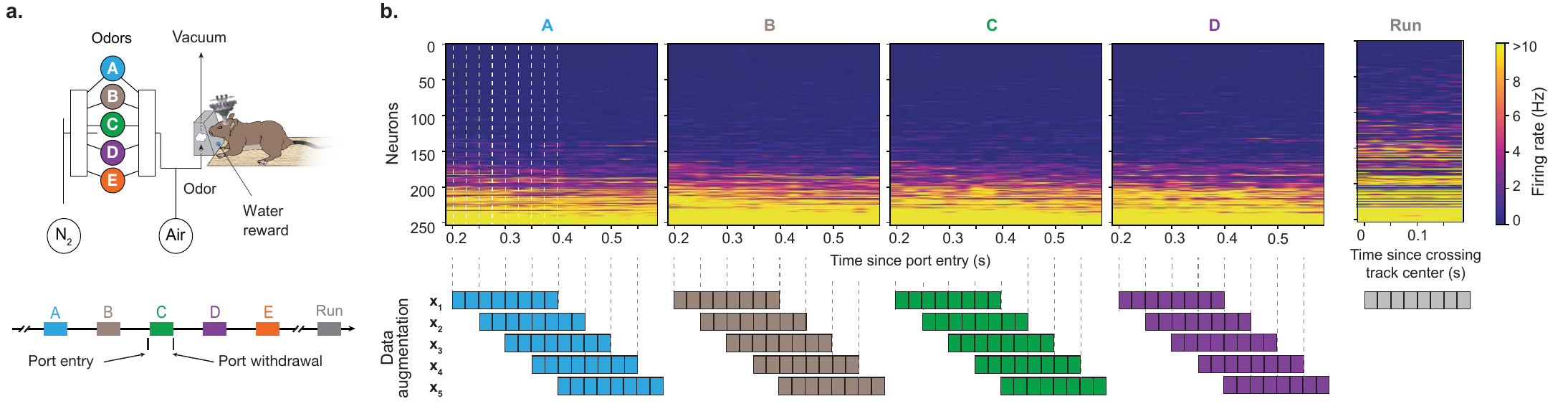}
    \caption{Overview of the behavioral task and neural activity data processing. (a) Schematic of the sequence memory task in which each presentation of the sequence of non-spatial stimuli (odors ABCDE) is followed by a running period (Run). For each odor presentation, subjects (rats) must perform the correct response to receive a water reward. (b) Visualization of neural spiking activity aggregated from all five rats across different experimental states (odors A--D and running) and the subsequent data augmentation and binning process.}
    \label{fig:rat-raw-data}
\end{figure}

The following analysis utilizes neural spiking activities of the rats during correctly identified InSeq trials. To maintain class balance, we focused on the first four stimuli in each sequence, as incorrect responses terminated the sequence and sample sizes decreased with sequence position. Specifically, we trained the model on the window of 0.2 - 0.6s following port entry to decode the odor identity. These training periods were selected based on prior research identifying them as the most representative intervals for stimulus processing \citep{shahbaba2022hippocampal}. To make the model training more robust, we utilize a data augmentation process by applying a sliding window procedure during the training periods, treating each window as an independent sample.  The spiking activity was then discretized into adjacent bins to estimate firing rates and mitigate data sparsity. The resulting preprocessed input data, $X$, comprises three dimensions: number of samples, number of time bins, and number of neurons. \Cref{fig:rat-raw-data} provides a visualization of the memory task and the data processing steps.

Finally, we implemented Fourier Holographic Reduced Representations (FHRR) to encode these features into hypervectors representing temporal trajectories within the high-dimensional state space. We present the encoding details in \Cref{app:neuron-encoding}.

Results for one representative rat are presented in \Cref{tab:real-data}, while the comprehensive results for all five rats are available \Cref{tab:rat} in \cref{app:additional-results}. In this analysis, the ``run'' state data is utilized as the OOD class. Similar to the previous section, we report marginal coverage, average prediction set size, and AUC. Our CHDC-ratio and CHDC-discount methods consistently produce smaller, more efficient prediction sets. In addition, CHDC-discount demonstrates superior and robust performance in OOD detection. 

\section{Discussion}\label{sec:discussion}
In this work, we introduced ConformalHDC, a lightweight framework providing rigorous uncertainty quantification for Hyperdimensional Computing. By utilizing HDC's prototype-based architecture to build adaptive nonconformity scores, we achieve valid coverage while maintaining high-speed execution. 


The current ConformalHDC implementation relies on a split-conformal approach, which requires a distinct calibration set. While computationally efficient, this partitioning can be wasteful in small-data regimes, potentially leading to less representative class prototypes due to the reduced training size. Future extensions could incorporate more parsimonious data-use techniques such as full conformal inference~\citep{vovk2013transductive} or jackknife+~\citep{barber_cv+_2021}. These methods lead to more efficient data usage, allowing more data points for both prototype generation and quantile estimation, thereby enhancing the robustness of the resulting uncertainty bounds when $n$ is limited. 

Furthermore, the framework can be extended to online conformal prediction by incrementally updating prototypes and calibration quantiles as data streams arrive.

\clearpage
\printbibliography 

\clearpage
\appendix
\clearpage

\renewcommand{\thefigure}{\thesection.\arabic{figure}}
\renewcommand{\thetable}{\thesection.\arabic{table}}
\renewcommand{\thealgorithm}{\thesection.\arabic{algorithm}}

\setcounter{figure}{0}
\setcounter{table}{0}
\setcounter{algorithm}{0}

\section{Implementation Details}
\subsection{Standardization of Similarity Measures}\label{app:positive-similarity}
To ensure that the nonconformity scores defined in \Cref{sec:conformity-scores} are well-defined, specifically to maintain the probabilistic interpretation of the ratio terms and the convention of nonconformity score that the smaller scores means conformity, all similarity measurements $\delta(\cdot, \cdot)$ must be nonnegative. We adopt the following standardization procedures for the most common distance and similarity metrics used in the HDC literature.

\textbf{Cosine Similarity.} The raw cosine similarity $\mathrm{cos}(h_1,h_2)$ of two hypervectors ranges from $[-1, 1]$. To map this to the unit interval $[0, 1]$, we apply a linear transformation:
\begin{align*}
\delta_{\cos}(h_1, h_2) = \frac{\mathrm{cos}(h_1,h_2) + 1 }{2}.
\end{align*}
This normalization is a standard procedure in Vector Symbolic Architectures (VSA) to ensure scores remain non-negative without altering the relative ranking of class prototypes \citep{Schlegel2022-uy}.

\textbf{Euclidean Distance.} Unlike cosine similarity, Euclidean distance is a dissimilarity measure where larger values indicate lower correspondence. To convert this into a positive similarity measure, we employ an inverse $L_2$ distance mapping:
\begin{align*}
    \delta_{\text{Euc}}(h_1, h_2) = \frac{1}{\|h_1 - h_2\|_2},
\end{align*}
where $\|\cdot\|_2$ denotes the standard $L_2$ norm. In practice, one can add a small constant to the denominator in the case of an exact match where the distance is zero.

\textbf{Hamming Distance.} The Hamming distance $d_H(h_1, h_2)$ is frequently used with binary hypervectors $\mathcal{H} \in \{0, 1\}^d$ or $\mathcal{H} \in \{-1, +1\}^d$. It is defined as the number of positions at which the corresponding components of the two hypervectors are different:
\begin{align*}
d_H(h_1, h_2) = \sum_{i=1}^d \mathbb{I}(h_{1,i} \neq h_{2,i}),
\end{align*}
where $\mathbb{I}(\cdot)$ is the indicator function. Because $d_H$ is a count of disagreements, it is inherently non-negative ($d_H \in [0, d]$). Consequently, the non-negativity requirement is naturally satisfied without further transformation, allowing it to be used directly in the nonconformity frameworks described in \Cref{sec:conformity-scores}.

\subsection{Prototyping and encoding of benchmark datasets}\label{app:benchmark-encoding}
In this section, we detail the preprocessing and hyperdimensional encoding schemes applied to the MNIST, ISOLET, and European Languages datasets introduced in Section 5.1. For each dataset, the encoding function $\phi(\cdot)$ maps an input sample $X$ into a $d$-dimensional hypervector $H$. In our experiments, we utilized $d=10,000$.

\paragraph{Image Classification (MNIST).}
We utilize the the MNIST dataset from the \texttt{torchvision} library, comprising $28 \times 28$ grayscale images. The images are first flattened into vectors of size $p=784$. We apply a threshold of $0.5$ to binarize the pixel values, resulting in a feature vector $X \in \{0, 1\}^p$.

To encode the spatial structure, we generate a fixed, random bipolar position hypervector $P_j \in \{-1, +1\}^d$ for each pixel index $j \in \{1, \dots, p\}$. The encoding for an image $X$ is obtained by bundling the position hypervectors corresponding to active pixels:
$$H = \text{sign}\left( \sum_{j=1}^{p} X_j P_j \right)$$
where $X_j$ acts as a selector ($1$ if the pixel is active, $0$ otherwise), and the sign function is applied element-wise to return a bipolar hypervector \citep{manabat2019performance}. We use the normalized cosine similarity discussed in \Cref{app:positive-similarity} and class prototypes are constructed by summing the encoded hypervectors of all training samples belonging to a class and applying the sign function to the result.

\paragraph{Voice Recognition (ISOLET).}
We use the The ISOLET dataset from the \texttt{sklearn} library, which consists of 617 acoustic features extracted from 26 spoken letters (letters A-Z). To handle continuous values, we quantize each feature into $L=21$ discrete levels. We compute the global minimum and maximum for each feature across the training set to define linear bin edges, mapping features to integer indices $v_{j} \in \{0, \dots, L-1\}$ for $j \in \{1, \dots, 617\}$.

We adopt the identification-value chain encoding (VoiceHD \citep{imani2017voicehd}.) This utilizes two sets of hypervectors: (1) \textit{Item Memory (iM)}, consisting of random orthogonal hypervectors $ID_j \in \{-1, +1\}^d$ representing the identity of each feature index $j$; and (2) \textit{Continuous Item Memory (CiM)}, consisting of correlated level hypervectors $L_v \in \{-1, +1\}^d$ for each quantization level. The level hypervectors are constructed such that the Hamming distance between $L_u$ and $L_v$ is proportional to $|u - v|$. The sample is encoded by binding the feature identity with its corresponding level value and bundling across all features:
$$H = \text{sign}\left( \sum_{j=1}^{617} ID_j \otimes L_{v_j} \right)$$
where $\otimes$ denotes element-wise multiplication (XOR in binary space). As with MNIST, we use the normalized cosine similarity discussed in \Cref{app:positive-similarity} and class prototypes are formed by aggregating the encoded training samples and binarizing via the sign function.

\paragraph{Language Classification (European Languages).}
We use the The European Languages dataset from the \texttt{torchhd} library, which consists of text samples from 21 European languages. The input text is preprocessed by converting to lowercase and removing excess whitespace, restricting the input to a maximum sequence length of 128 characters. The vocabulary consists of the 26 Latin alphabet characters plus a token for spaces.

We utilize an $N$-gram encoding scheme (with $N=3$) to capture local sequential context \citep{joshi2016language}. Each character $c$ in the vocabulary is assigned a static random hypervector $S_c$. For a text sequence converted to hypervectors $S_{t}, S_{t+1}, \dots$, a trigram at position $t$ is encoded by binding permuted inputs: $G_t = S_t \otimes \rho(S_{t+1}) \otimes \rho^2(S_{t+2})$, where $\rho$ is a cyclic shift (permutation) operator. The final text representation is the superposition of all valid trigrams in the sequence, binarized by the sign function:
$$H = \text{sign}\left( \sum_{t} S_t \otimes \rho(S_{t+1}) \otimes \rho^2(S_{t+2}) \right)$$
We use the normalized cosine similarity in \Cref{app:positive-similarity}. During training, we accumulate the hypervectors of all samples for a given language. Unlike the binary prototypes used for MNIST and ISOLET, the final language prototypes are normalized real-valued vectors (accumulated sums normalized by the $L_2$ norm) to facilitate precise cosine similarity calculations during inference.

\subsection{Prototyping and Encoding of the Neural Dataset}\label{app:neuron-encoding}
As described in \cref{sec:exp-neuron}, we first performed a sliding window procedure during the training period, which is selected 200-600 ms after the odor delivery signal. The sliding window size is 200 ms, and the step size is 50 ms, yielding 5 samples per trial. Then, we transformed the spike data during each sample window into time-binned firing rate estimates by slicing the recordings into adjacent bins and computing the total spike counts normalized by bin duration (bin size 25 ms, resulting in 8 bins per sample). 

We use Fourier Holographic Reduced Representations (FHRR), to encode input data into hypervectors as complex numbers in the unit circle \citep{377968}. An FHRR hypervector is of the form $H = e^{i \theta} \in \mathbb{C}^d$, where $\theta \in \mathbb{R}^{d}$ and $i = \sqrt{-1}$ is the imaginary number. 
We defined the similarity measure $\delta : \mathcal{H}\times \mathcal{H} \mapsto [0,1]$ as normalized cosine similarity for complex hypervectors:
\begin{align}\label{eq:sim-complex-cosine}
    \delta(h_1,h_2) =  \frac{1}{2} \paren*{\frac{\Re [h_1^T \bar{h}_2]}{d}+1},
\end{align}
where $\bar{h}$ is the complex conjugate of $h$ and $\Re[h]$ is the real part of $h$.

\begin{figure*}[!htb]
    \centering
    \includegraphics[width=\linewidth]{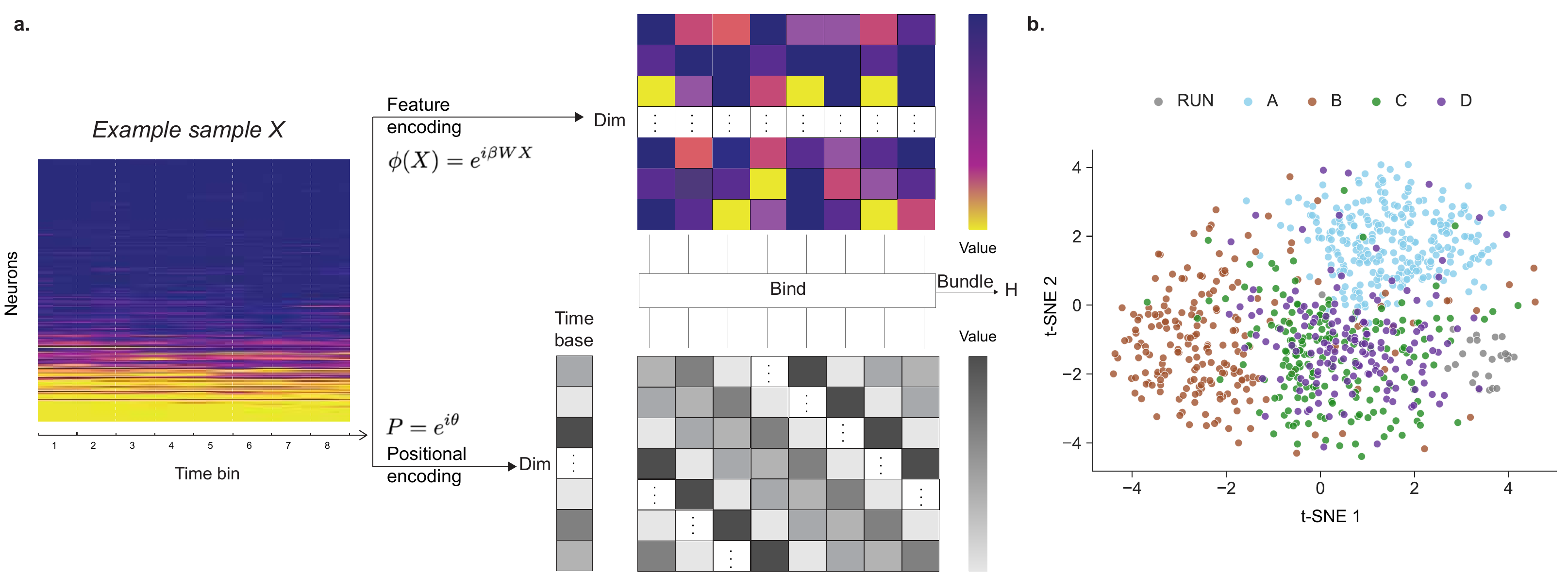}
    \caption{Encoding of the spike activities. (a) The full encoding pipeline using FHRR. (b) A T-SNE visualization of the encoded hypervectors for different states using data from rat Superchris.}
    \label{fig:rat-encoding}
\end{figure*}

The left panel in \Cref{fig:rat-encoding} visualizes the encoding steps. Each sample $X \in \mathbb{R}^{p\times t}$ represents neural spike activities, where $p$ is the number of features, i.e., number of neurons and $t$ is the number of time bins. For each time bin $j \in [t]$, we denote the neural activities as $X_{\cdot,j} \in \mathbb{R}^p$, namely, the $j$-th column of sample $X$. To transform this continuous column into hypervectors, we employ Fractional Power Encoding (FPE). FPE aligns with Random Fourier Features (RFF), enabling the model to capture nonlinear relationships within the feature space governed by the choice of kernel \citep{rahimi2007random}. We define the encoding function as $\phi: \mathbb{R}^p \rightarrow \mathbb{C}^d$, with 
\begin{align*}
    \phi(X_{\cdot,j}) = e^{i \beta W X_{\cdot,j}}, \text{ with projection matrix } W\in \mathbb{R}^{d\times p},
\end{align*}
here we adopt a Radial Basis Function (RBF) kernel by sampling each row $W_{k,\cdot} \sim \mathcal{N}(0, I_p)$ from a multivariate Gaussian distribution with an identity covariance matrix $I_p \in \mathbb{R}^{p \times p}$. Under this choice of kernel, the complex cosine similarity defined in \Cref{eq:sim-complex-cosine} approximates the Euclidean distance in the original feature space $\mathbb{R}^p$.
Here $\beta$ is a fractional power factor, that tunes how similar two hypervectors can be in the hyperplane. For all experiments in \Cref{sec:exp-neuron}, we use $\beta=0.3$.

To represent temporal information across all time bins $j \in [t]$, we first sample a positional (time) indicator hypervector $P = e^{i\theta}$, where each element of $\theta \in \mathbb{R}^d$ is drawn from a uniform distribution $\mathrm{Unif}(0, 2\pi)$. Temporal order is then encoded via cyclic permutations of $P$. For each time bin, we bind the feature encoding to its corresponding temporal indicator using element-wise multiplication, $\phi(X_{\cdot,j}) \otimes \rho^j(P)$, and aggregate these across the entire sequence through bundling:
\begin{align*}
H = \sum_{j=1}^t \phi(X_{\cdot,j}) \otimes \rho^j(P).
\end{align*}
During training, class prototypes are constructed by accumulating the resulting hypervectors for each respective class.


\subsection{Synthetic Data Generation}\label{app:synthetic-3class}
To evaluate the proposed ConformalHDC framework, we design a synthetic $K$-class classification problem ($K=3$) within feature space $\mathcal{X} \subseteq \mathbb{R}^p$ where $p=2$. The data geometry is specifically constructed to challenge the classifier's robustness under varying degrees of class overlap and heteroscedasticity. The in-distribution data are sampled from Gaussian blobs with centers arranged in an equilateral triangle:
\begin{align*}
    \mu_1 = [0, 0]^\top, \quad \mu_2 = [L, 0]^\top, \quad \mu_3 = [L/2, (\sqrt{3}/2)L]^\top,
\end{align*}
where $L=4\sqrt{p}$ defines the equidistant separation between class centers. For each class $y \in \{1, 2, 3\}$, we sample $n_y$ observations from an isotropic multivariate Gaussian distribution:
\begin{align*}
    X \mid Y=y \sim \mathcal{N}(\mu_y, \sigma_y^2 I),
\end{align*}
where $I$ denotes the identity matrix. To evaluate the ability of our nonconformity scores to adapt to classes with varying spatial densities, we introduce heteroscedasticity by fixing $\sigma_0 = 1.0$ and $\sigma_1 = 2.0$, while sweeping the standard deviation of the third class, $\sigma_3$, across the range $[3.0, 5.0]$ in increments of $0.25$. This setup provides an intuitive testbed: as $\sigma_3$ increases, class 3 becomes more dispersed and harder to distinguish, making valid uncertainty quantification more desirable.

Furthermore, we assess the framework's capability for anomaly detection and prediction deferral by generating an out-of-distribution (OOD) cluster. Let $\bar{\mu} = \frac{1}{K}\sum_{y=1}^K \mu_y$ denote the centroid of the in-distribution prototypes. The OOD cluster is sampled from $\mathcal{N}(\mu_{\text{ood}}, I)$, where the center $\mu_{\text{ood}}$ is moved away from the in-distribution centroid:
\begin{align*}
    \mu_{\text{ood}} = \bar{\mu} + 1.8L[0, -1]^\top.
\end{align*}


\section{Computational Complexity Analysis}\label{app:complexity}
This section evaluates the computational overhead of the ConformalHDC framework. In comparison to standard HDC pipelines, the cost introduced by conformalization is minimal. This is because the calibration set size  $n_{\text{cal}}$ is typically much smaller than the training set size $n_{\text{train}}$, and the evaluation of nonconformity scores leverages the existing prototype-based architecture, requiring only basic arithmetic operations.
 
To keep this analysis model-agnostic, we denote the base HDC training cost (encompassing encoding and prototyping) as $\mathcal{O}(T)$. While $T$ varies by implementation, it is generally a function of the training size $n_{\text{train}}$, the hypervector dimension $d$, and the raw feature dimension $p$. For example, the encoding and training process for the neural data described in \Cref{app:neuron-encoding} scales as $\mathcal{O}(n_{\text{train}}dp)$.

Table~\ref{tab:complexity} summarizes the costs for our proposed algorithms when processing $m$ test points. Across all variants, the asymptotic complexity remains $\mathcal{O}(T + n_{\text{cal}}Kd + mKd)$. This includes a one-time calibration cost of $\mathcal{O}(T + n_{\text{cal}}Kd)$ to compute calibration scores and empirical quantiles, followed by an inference cost of $\mathcal{O}(mKd)$ to generate prediction sets. Detailed breakdowns for each procedure can be found in \Cref{app:cost-marginal} to \Cref{app:cost-point}.

\begin{table}[h]
\centering
\caption{Computational complexity of ConformalHDC algorithms. $T$ is the training cost, $d$ the hypervector dimension, $K$ the number of classes, $n_{\text{cal}}$ the calibration size, and $m$ the number of test points.}
\label{tab:complexity}
\resizebox{\textwidth}{!}{
\begin{tabular}{@{}llll@{}}
\toprule
\textbf{Algorithm} & \textbf{Calibration (One-time)} & \textbf{Inference ($m$ points)} & \textbf{Total Cost} \\ \midrule
Set-valued marg. (\Cref{alg:set-conformalHDC-marginal}) & $\mathcal{O}(T + n_{\text{cal}}Kd)$ & $\mathcal{O}(mKd)$ & $\mathcal{O}(T + n_{\text{cal}}Kd + mKd)$ \\
Set-valued cond. (\Cref{alg:set-conformalHDC-cond}) & $\mathcal{O}(T + n_{\text{cal}}Kd)$ & $\mathcal{O}(mKd)$ & $\mathcal{O}(T + n_{\text{cal}}Kd + mKd)$ \\
Point-valued (\Cref{alg:point-conformalHDC}) & $\mathcal{O}(T + n_{\text{cal}}Kd)$ & $\mathcal{O}(mKd)$ & $\mathcal{O}(T + n_{\text{cal}}Kd + mKd)$ \\ \bottomrule
\end{tabular}
}
\end{table}

\subsection{Computational Cost Analysis of Algorithm~\ref{alg:set-conformalHDC-marginal}}\label{app:cost-marginal}
The cost of producing marginal prediction sets for $m$ test points is $\mathcal{O}(T + n_{\text{cal}}Kd + mKd)$, as shown below.
\begin{itemize} \setlength\itemsep{0em}
    \item Training the HDC model (encoding and prototype generation) on $\mathcal{D}_{\text{train}}$ has cost $\mathcal{O}(T)$.
    \item Computing the nonconformity scores $\widehat{S}_i$ for all $i \in \mathcal{D}_{\text{cal}}$ requires $n_{\text{cal}}$ similarity computations against $K$ prototypes, each of dimension $d$, resulting in $\mathcal{O}(n_{\text{cal}}Kd)$.
    \item Finding the $(1-\alpha)$ empirical quantile $\widehat{Q}_{1-\alpha}$ from the calibration scores has cost $\mathcal{O}(n_{\text{cal}})$ using a selection algorithm.
    \item For $m$ test points, we first encode each point and compute its similarity to $K$ prototypes ($\mathcal{O}(mKd)$). Given the similarities, evaluating the nonconformity score for each possible label $y \in [K]$ and comparing against $\widehat{Q}_{1-\alpha}$ in \eqref{eq:conformal-pset} takes $\mathcal{O}(mK)$, leading to $\mathcal{O}(mKd)$.
\end{itemize}
Therefore, the overall cost is $\mathcal{O}(T + n_{\text{cal}}Kd + mKd)$.

\subsection{Computational Cost Analysis of Algorithm~\ref{alg:set-conformalHDC-cond}}\label{app:cost-conditional}
The cost of producing label-conditional prediction sets for $m$ test points is $\mathcal{O}(T + n_{\text{cal}}Kd + mK^2d)$, as shown below.
\begin{itemize} \setlength\itemsep{0em}
    \item Training the HDC model (encoding and prototype generation) on $\mathcal{D}_{\text{train}}$ has cost $\mathcal{O}(T)$.
    \item Computing the nonconformity scores $\widehat{S}_i$ for all $i \in \mathcal{D}_{\text{cal}}$ requires $n_{\text{cal}}$ similarity computations against $K$ prototypes, each of dimension $d$, resulting in $\mathcal{O}(n_{\text{cal}}Kd)$.
    \item Computing $K$ separate quantiles $\widehat{Q}^y_{1-\alpha}$ involves partitioning the $n_{\text{cal}}$ scores and finding a quantile for each group, taking $\mathcal{O}(n_{\text{cal}}K)$.
    \item The inference cost for $m$ test points is $\mathcal{O}(mKd)$ because each class label $y \in [K]$ is checked against its label-specific threshold $\widehat{Q}^y_{1-\alpha}$ using the same scoring logic as the marginal case.
\end{itemize}
Therefore, the overall cost is $\mathcal{O}(T + n_{\text{cal}}Kd + mKd)$.

\subsection{Computational Cost Analysis of Algorithm~\ref{alg:point-conformalHDC}}\label{app:cost-point}
The cost of producing point-valued predictions for $m$ test points is $\mathcal{O}(T + n_{\text{cal}}Kd + mKd)$, as shown below.
\begin{itemize} \setlength\itemsep{0em}
    \item This algorithm serves as a post-processing refinement of the prediction sets produced by Algorithm~\ref{alg:set-conformalHDC-marginal} or \ref{alg:set-conformalHDC-cond}.
    \item For each test point, if the prediction set size $|\widehat{C}_{\alpha}| > 1$ or $|\widehat{C}_{\alpha}| = 0$, the cost of identifying the label with the minimum nonconformity score among the candidates is $\mathcal{O}(K)$.
    \item Therefore, trimming $m$ prediction sets takes $\mathcal{O}(mK)$.
\end{itemize}
Therefore, Algorithm~\ref{alg:point-conformalHDC} has a total time complexity of $\mathcal{O}(T + n_{\text{cal}}Kd + mKd)$, maintaining the same asymptotic efficiency as the set-valued procedures.

\section{Review of conformity scores for classification} \label{app:class-scores}

This section reviews the relevant background on the adaptive conformity scores for classification developed by \citet{romano2020classification}.
For any $x \in \mathcal{X}$ and $y \in [K]$, let $\widehat{\pi}_{y}(x)$ denote any (possibly very inaccurate) estimate of the true $\mathbb{P}[ Y = y \mid X =x]$ corresponding to the unknown data-generating distribution. Concretely, a typical choice of $\widehat{\pi}$ may be given by the output of the final softmax layer of a neural network classifier, for example.
 For any $x \in \mathcal{X}$ and $\tau \in [0,1]$, define the \emph{generalized conditional quantile} function $L$, with input $x, \widehat{\pi}, \tau$, as:
\begin{align} \label{eq:oracle-threshold}
  L(x; \widehat{\pi}, \tau) & = \min \{ k \in [K] \ : \ \widehat{\pi}_{(1)}(x) + \widehat{\pi}_{(2)}(x) + \ldots + \widehat{\pi}_{(k)}(x) \geq \tau \},
  \end{align}
where $\widehat{\pi}_{(1)}(x) \leq \widehat{\pi}_{(2)}(x) \leq \ldots \widehat{\pi}_{(K)}(x)$ are the order statistics of $\widehat{\pi}_{1}(x) \leq \widehat{\pi}_{2}(x) \leq \ldots \widehat{\pi}_{K}(x)$.
Intuitively, $L(x; \widehat{\pi}, \tau)$ gives the size of the smallest possible subset of labels whose cumulative probability mass according to $\widehat{\pi}$ is at least $\tau$.
Define also a function $\mathcal{S}$ with input $x$, $u \in (0,1)$, $\widehat{\pi}$, and $\tau$ that computes the set of most likely labels up to (but possibly excluding) the one identified by $L(x; \widehat{\pi}, \tau)$:
\begin{align} \label{eq:define-S}
    \mathcal{S}(x, u ; \widehat{\pi}, \tau) & =
    \begin{cases}
    \text{ `$y$' indices of the $L(x ; \widehat{\pi},\tau)-1$ largest $\widehat{\pi}_{y}(x)$},
    & \text{ if } u \leq V(x ; \widehat{\pi},\tau) , \\
    \text{ `$y$' indices of the $L(x ; \widehat{\pi},\tau)$ largest $\widehat{\pi}_{y}(x)$},
    & \text{ otherwise},
    \end{cases}
\end{align}
where
\begin{align*}
    V(x; \widehat{\pi}, \tau) & =  \frac{1}{\widehat{\pi}_{(L(x ; \widehat{\pi}, \tau))}(x)} \left[\sum_{k=1}^{L(x ; \widehat{\pi}, \tau)} \widehat{\pi}_{(k)}(x) - \tau \right].
\end{align*}
Then, define the \textit{generalized inverse quantile} conformity score function $s$, with input $x,y,u;\widehat{\pi}$, as:
\begin{align} \label{eq:define-scores}
    s(x,y,u;\widehat{\pi}) & = \min \left\{ \tau \in [0,1] : y \in \mathcal{S}(x, u ; \widehat{\pi}, \tau) \right\}.
\end{align}
Intuitively, $s(x,y,u;\widehat{\pi})$ is the smallest value of $\tau$ for which the set $\mathcal{S}(x, u ; \widehat{\pi}, \tau)$ contains the label $y$.
Finally, the conformity score for a data point $(X_i,Y_i)$ is given by:
\begin{align}
  \widehat{S}_i
  & = s(X_i,Y_i,U_i;\widehat{\pi}),
\end{align}
where $U_i$ is a uniform random variable independent of anything else. Note that this can also be equivalently written more explicitly as:
\begin{align}\label{eq:inverse_quant_conformity_score}
  \widehat{S}_i
  & = \widehat{\pi}_{(1)}(X_i) + \widehat{\pi}_{(2)}(X_i) + \ldots + \widehat{\pi}_{(r(Y_i,\widehat{\pi}(X_i)))}(X_i) - U_i\cdot \widehat{\pi}_{(r(Y_i,\widehat{\pi}(X_i)))}(X_i),
\end{align}
where $r(Y_i,\widehat{\pi}(X_i))$ is the rank of $Y_i$ among the possible labels $y \in [K]$ based on $\widehat{\pi}_y(X_i)$, so that $r(y,\widehat{\pi}(X_i))=1$ if $\widehat{\pi}_{y}(X_i) = \widehat{\pi}_{(1)}(X_i)$. One can convert the conformity score in \eqref{eq:inverse_quant_conformity_score} into nonconformity score by negation.

\section{Mathematical Proofs}\label{app:proofs}
\begin{proof}[Proof of \Cref{thm:marginal}]
    By construction of the prediction set $\widehat{C}_{\alpha}(X_{n+1})$ in \eqref{eq:conformal-pset}, we have
    \begin{align*}
        \mathbb{P}\cbrac{Y_{n+1} \in \widehat{C}_{\alpha}(X_{n+1})} =  \mathbb{P}\cbrac{\widehat{S}(X_{n+1}, Y_{n+1}) \leq \widehat{Q}_{1-\alpha}(\{\widehat{S}_i\}_{i \in \mathcal{D}_{\text{cal}}})}.
    \end{align*}
   The initial assumption of exchangeability for the sequence $\cbrac{(X_i, Y_i)}_{i=1}^{n+1}$ implies that the nonconformity scores derived from the calibration set and test point, $\cbrac{s_i}_{i \in \mathcal{D}_{\text{cal}} \cup \{n+1\}}$, are also exchangeable. The rest of the proof for both the upper and lower bounds then follows directly from the standard quantile inflation argument in Lemma 2 of~\citet{romano2019conformalized}. 
\end{proof}

\begin{proof}[Proof of \Cref{thm:conditional}]
    By construction of $\widehat{C}_{\alpha}(X_{n+1})$ in \eqref{eq:conditional-pset}, for any $y\in[K]$, we have
    \begin{align*}
        \mathbb{P}\cbrac{Y_{n+1} \in \widehat{C}_{\alpha}(X_{n+1}) \mid Y_{n+1}=y} =  \mathbb{P}\cbrac{\widehat{S}(X_{n+1}, Y_{n+1}) \leq \widehat{Q}^y_{1-\alpha}(\{\widehat{S}_i\}_{i \in \mathcal{D}^y_{\text{cal}}})\mid Y_{n+1}=y}.
    \end{align*}
    the results for both the upper and lower bounds is derived by applying the proof of \Cref{thm:marginal} to the subset $\mathcal{D}^y_{\text{cal}}$, since condition on $Y_{n+1}=y$, $(X_{n+1}, Y_{n+1})$ is exchangeable with $\mathcal{D}^y_{\text{cal}}$.
\end{proof}

\section{Additional Numerical Results}\label{app:additional-results}
In this section, we provide comprehensive numerical results and extended performance metrics for the experiments discussed in the main text. These tables include detailed breakdowns of marginal coverage, prediction set sizes, accuracy and AUC.

Specifically, \Cref{tab:3-class} presents the full results for the synthetic experiment introduced in \Cref{sec:synthetic-exp}. Following this, \Cref{tab:mnist} through \Cref{tab:isolet} provide the complete performance evaluations for the real-world benchmark datasets including MNIST, ISOLET, and Languages. \Cref{tab:rat} collects the results for all five subjects from the neural decoding application. For all tables, results are reported as mean values across multiple independent repetitions, with standard errors provided in parentheses.

\begin{table}[ht]
    \centering
    \footnotesize                       
    \renewcommand{\arraystretch}{1} 
    \setlength{\tabcolsep}{12pt}
    \begin{tabular}{|ll|cc|c|c|}
   \hline
 & & \multicolumn{2}{c|}{\textbf{Set-Valued}} & \multicolumn{1}{c|}{\textbf{Point}} & \multicolumn{1}{c|}{\textbf{OOD}} \\
\cline{3-6}
\textbf{Sigma} & \textbf{Method} & \textbf{Coverage} & \textbf{Size} & \textbf{Accuracy} & \textbf{AUC} \\
\hline
\multirow{6}{*}{3} & HDC & 0.861 (0.001) & 1.000 (0.000) & 0.861 (0.001) & - \\ 
   & Inv. quantile & 0.901 (0.001) & 1.448 (0.004) & 0.862 (0.001) & 0.391 (0.001) \\ 
   & Penalized & 0.899 (0.001) & 1.246 (0.004) & 0.876 (0.001) & 0.884 (0.001) \\ 
   & Similarity & 0.900 (0.001) & 1.444 (0.005) & 0.875 (0.001) & 0.997 (0.000) \\ 
   & CHDC-ratio & 0.900 (0.001) & 1.103 (0.002) & 0.875 (0.001) & 0.986 (0.000) \\ 
   & CHDC-discount & 0.900 (0.001) & 1.159 (0.003) & 0.877 (0.001) & 0.998 (0.000) \\ 
   \hline
\multirow{6}{*}{3.25} & HDC & 0.851 (0.001) & 1.000 (0.000) & 0.851 (0.001) & - \\ 
   & Inv. quantile & 0.901 (0.001) & 1.488 (0.004) & 0.851 (0.001) & 0.392 (0.001) \\ 
   & Penalized & 0.899 (0.001) & 1.303 (0.005) & 0.868 (0.001) & 0.883 (0.001) \\ 
   & Similarity & 0.901 (0.001) & 1.569 (0.006) & 0.865 (0.001) & 0.995 (0.000) \\ 
   & CHDC-ratio & 0.898 (0.001) & 1.131 (0.003) & 0.866 (0.001) & 0.986 (0.000) \\ 
   & CHDC-discount & 0.899 (0.001) & 1.201 (0.003) & 0.869 (0.001) & 0.997 (0.000) \\ 
   \hline
\multirow{6}{*}{3.5} & HDC & 0.846 (0.001) & 1.000 (0.000) & 0.846 (0.001) & - \\ 
   & Inv. quantile & 0.901 (0.001) & 1.511 (0.004) & 0.846 (0.001) & 0.393 (0.001) \\ 
   & Penalized & 0.902 (0.001) & 1.357 (0.005) & 0.863 (0.001) & 0.885 (0.001) \\ 
   & Similarity & 0.901 (0.001) & 1.721 (0.006) & 0.859 (0.002) & 0.991 (0.000) \\ 
   & CHDC-ratio & 0.901 (0.001) & 1.166 (0.003) & 0.862 (0.001) & 0.986 (0.000) \\ 
   & CHDC-discount & 0.901 (0.001) & 1.243 (0.003) & 0.866 (0.001) & 0.994 (0.000) \\ 
   \hline
\multirow{6}{*}{3.75} & HDC & 0.839 (0.002) & 1.000 (0.000) & 0.839 (0.002) & - \\ 
   & Inv. quantile & 0.902 (0.001) & 1.543 (0.004) & 0.838 (0.002) & 0.395 (0.001) \\ 
   & Penalized & 0.902 (0.002) & 1.397 (0.005) & 0.857 (0.002) & 0.885 (0.001) \\ 
   & Similarity & 0.899 (0.001) & 1.853 (0.007) & 0.850 (0.002) & 0.987 (0.000) \\ 
   & CHDC-ratio & 0.901 (0.002) & 1.195 (0.003) & 0.856 (0.002) & 0.985 (0.000) \\ 
   & CHDC-discount & 0.900 (0.002) & 1.282 (0.004) & 0.861 (0.002) & 0.991 (0.000) \\ 
   \hline
\multirow{6}{*}{4} & HDC & 0.830 (0.002) & 1.000 (0.000) & 0.830 (0.002) & - \\ 
   & Inv. quantile & 0.900 (0.001) & 1.563 (0.004) & 0.830 (0.002) & 0.395 (0.001) \\ 
   & Penalized & 0.900 (0.001) & 1.416 (0.005) & 0.854 (0.001) & 0.886 (0.001) \\ 
   & Similarity & 0.899 (0.001) & 2.010 (0.007) & 0.843 (0.002) & 0.982 (0.000) \\ 
   & CHDC-ratio & 0.900 (0.001) & 1.221 (0.003) & 0.852 (0.001) & 0.984 (0.000) \\ 
   & CHDC-discount & 0.900 (0.001) & 1.327 (0.003) & 0.856 (0.001) & 0.987 (0.000) \\ 
   \hline
\multirow{6}{*}{4.25} & HDC & 0.826 (0.001) & 1.000 (0.000) & 0.826 (0.001) & - \\ 
   & Inv. quantile & 0.901 (0.001) & 1.584 (0.004) & 0.824 (0.001) & 0.396 (0.001) \\ 
   & Penalized & 0.898 (0.001) & 1.437 (0.005) & 0.849 (0.001) & 0.885 (0.001) \\ 
   & Similarity & 0.900 (0.001) & 2.145 (0.007) & 0.843 (0.001) & 0.977 (0.001) \\ 
   & CHDC-ratio & 0.900 (0.001) & 1.247 (0.004) & 0.849 (0.001) & 0.983 (0.000) \\ 
   & CHDC-discount & 0.901 (0.001) & 1.372 (0.004) & 0.854 (0.001) & 0.983 (0.001) \\ 
   \hline
\multirow{6}{*}{4.5} & HDC & 0.821 (0.002) & 1.000 (0.000) & 0.821 (0.002) & - \\ 
   & Inv. quantile & 0.901 (0.001) & 1.607 (0.003) & 0.817 (0.002) & 0.397 (0.001) \\ 
   & Penalized & 0.900 (0.001) & 1.462 (0.004) & 0.844 (0.001) & 0.888 (0.001) \\ 
   & Similarity & 0.898 (0.001) & 2.268 (0.007) & 0.836 (0.002) & 0.970 (0.001) \\ 
   & CHDC-ratio & 0.901 (0.001) & 1.280 (0.003) & 0.843 (0.002) & 0.981 (0.000) \\ 
   & CHDC-discount & 0.900 (0.001) & 1.417 (0.004) & 0.849 (0.001) & 0.977 (0.001) \\ 
   \hline
\multirow{6}{*}{4.75} & HDC & 0.815 (0.001) & 1.000 (0.000) & 0.815 (0.001) & - \\ 
   & Inv. quantile & 0.899 (0.001) & 1.616 (0.004) & 0.812 (0.001) & 0.397 (0.001) \\ 
   & Penalized & 0.899 (0.001) & 1.456 (0.004) & 0.840 (0.001) & 0.892 (0.001) \\ 
   & Similarity & 0.899 (0.001) & 2.366 (0.005) & 0.835 (0.002) & 0.964 (0.001) \\ 
   & CHDC-ratio & 0.899 (0.001) & 1.295 (0.004) & 0.841 (0.001) & 0.980 (0.000) \\ 
   & CHDC-discount & 0.898 (0.001) & 1.461 (0.004) & 0.846 (0.001) & 0.972 (0.001) \\ 
   \hline
\end{tabular}

    \caption{Full results for synthetic experiment with alpha=$0.1$. In-distribution and out-of-distribution data are sampled according to the distributions described in \Cref{app:synthetic-3class}. For all conformal methods, the in-distribution data is partitioned into training ($40\%$), calibration ($50\%$), and testing ($10\%$) sets. For the standard HDC classifier, the training and calibration folds are combined for model training, as it does not require a separate calibration step. Results are averaged over 100 repetitions. Values in parentheses report the standard errors. }
    \label{tab:3-class}
\end{table}

\begin{table}[ht]
    \centering
    \small                            
    \renewcommand{\arraystretch}{1.03} 
    \setlength{\tabcolsep}{12pt}
    \begin{tabular}{|l|cc|c|c|}
   \hline
 & \multicolumn{2}{c|}{\textbf{Set-Valued}} & \multicolumn{1}{c|}{\textbf{Point}} & \multicolumn{1}{c|}{\textbf{OOD}} \\
\cline{2-5}
\textbf{Method} & \textbf{Coverage} & \textbf{Size} & \textbf{Accuracy} & \textbf{AUC} \\
\hline
HDC & 0.862 (0.001) & 1.000 (0.000) & 0.862 (0.001) & - \\ 
  Inv. quantile & 0.950 (0.001) & 2.002 (0.007) & 0.862 (0.001) & 0.483 (0.000) \\ 
  Penalized & 0.950 (0.001) & 4.218 (0.007) & 0.858 (0.001) & 0.655 (0.001) \\ 
  Similarity & 0.951 (0.001) & 3.370 (0.009) & 0.863 (0.001) & 0.745 (0.001) \\ 
  CHDC-ratio & 0.950 (0.001) & 1.770 (0.005) & 0.863 (0.001) & 0.873 (0.001) \\ 
  CHDC-discount & 0.950 (0.001) & 2.425 (0.008) & 0.864 (0.001) & 0.815 (0.001) \\ 
   \hline
\end{tabular}

    \caption{Full results for MNIST experiment with alpha=$0.05$. Among 10 handwritten digits, last four digits, i.e., $\cbrac{6,7,8,9}$ are treated as OOD classes. For all conformal methods, the in-distribution data is split into training ($80\%$), calibration
    ($15\%$), and testing ($5\%$) sets. Other details remain the same as in \Cref{tab:3-class}.}
    \label{tab:mnist}
\end{table}

\begin{table}[ht]
    \centering
    \small                            
    \renewcommand{\arraystretch}{1.03} 
    \setlength{\tabcolsep}{12pt}
    \begin{tabular}{|l|cc|c|c|}
   \hline
 & \multicolumn{2}{c|}{\textbf{Set-Valued}} & \multicolumn{1}{c|}{\textbf{Point}} & \multicolumn{1}{c|}{\textbf{OOD}} \\
\cline{2-5}
\textbf{Method} & \textbf{Coverage} & \textbf{Size} & \textbf{Accuracy} & \textbf{AUC} \\
\hline
HDC & 0.953 (0.0008) & 1.000 (0.0000) & 0.953 (0.0008) & - \\ 
  Inv. quantile & 0.990 (0.0004) & 2.514 (0.0140) & 0.953 (0.0008) & 0.472 (0.0001) \\ 
  Penalized & 0.990 (0.0003) & 16.392 (0.0203) & 0.947 (0.0008) & 0.294 (0.0019) \\ 
  Similarity & 0.990 (0.0003) & 9.712 (0.0436) & 0.952 (0.0008) & 0.960 (0.0005) \\ 
  CHDC-ratio & 0.990 (0.0004) & 3.291 (0.0113) & 0.953 (0.0008) & 0.982 (0.0003) \\ 
  CHDC-discount & 0.990 (0.0003) & 5.354 (0.0246) & 0.953 (0.0008) & 0.973 (0.0004) \\ 
   \hline
\end{tabular}

    \caption{Full results for Languages experiment with alpha=$0.01$. Among the 21 European languages, Indo-European languages are treated as in-distribution, while non-Indo-European languages, i.e., \{Finnish, Estonian, Hungarian\} are treated as OOD classes. For all conformal methods, the in-distribution data is partitioned into training ($75\%$), calibration ($22.5\%$), and testing ($2.5\%$) sets. Other details remain the same as in \Cref{tab:3-class}. }
    \label{tab:languages}
\end{table}

\begin{table}[ht]
    \centering
    \small                            
    \renewcommand{\arraystretch}{1.03} 
    \setlength{\tabcolsep}{12pt}
    \begin{tabular}{|l|cc|c|c|}
   \hline
 & \multicolumn{2}{c|}{\textbf{Set-Valued}} & \multicolumn{1}{c|}{\textbf{Point}} & \multicolumn{1}{c|}{\textbf{OOD}} \\
\cline{2-5}
\textbf{Method} & \textbf{Coverage} & \textbf{Size} & \textbf{Accuracy} & \textbf{AUC} \\
\hline
HDC & 0.889 (0.002) & 1.000 (0.000) & 0.889 (0.002) & - \\ 
  Inv. quantile & 0.980 (0.001) & 3.378 (0.027) & 0.883 (0.002) & 0.482 (0.000) \\ 
  Penalized & 0.980 (0.001) & 21.344 (0.025) & 0.872 (0.002) & 0.536 (0.002) \\ 
  Similarity & 0.980 (0.001) & 7.428 (0.052) & 0.883 (0.002) & 0.759 (0.002) \\ 
  CHDC-ratio & 0.979 (0.001) & 4.429 (0.021) & 0.882 (0.002) & 0.747 (0.001) \\ 
  CHDC-discount & 0.980 (0.001) & 5.258 (0.037) & 0.883 (0.002) & 0.778 (0.001) \\ 
   \hline
\end{tabular}

    \caption{Full results for ISOLET experiment with alpha=$0.02$. Among the 26 English alphabet letters, the last four letters, i.e., $\cbrac{W, X, Y, Z}$ are treated as OOD classes. For all conformal methods, the in-distribution data is split into training ($57\%$), calibration ($38\%$), and testing ($5\%$) sets. Other details remain the same as in \Cref{tab:3-class}.}
    \label{tab:isolet}
\end{table}

\begin{table}[ht]
    \centering
    \small                            
    \renewcommand{\arraystretch}{1.15} 
    \setlength{\tabcolsep}{12pt}
    \begin{tabular}{|ll|cc|c|c|}
   \hline
 & & \multicolumn{2}{c|}{\textbf{Set-Valued}} & \multicolumn{1}{c|}{\textbf{Point}} & \multicolumn{1}{c|}{\textbf{OOD}} \\
\cline{3-6}
\textbf{Rat} & \textbf{Method} & \textbf{Coverage} & \textbf{Size} & \textbf{Accuracy} & \textbf{AUC} \\
\hline
\multirow{7}{*}{Barat} & HDC (train) & 0.472 (0.004) & 1.000 (0.000) & 0.472 (0.004) & - \\ 
   & HDC & 0.492 (0.004) & 1.000 (0.000) & 0.492 (0.004) & - \\ 
   & Inv. quantile & 0.794 (0.004) & 2.403 (0.008) & 0.472 (0.004) & 0.605 (0.000) \\ 
   & Penalized & 0.794 (0.004) & 3.003 (0.016) & 0.460 (0.004) & 0.011 (0.001) \\ 
   & Similarity & 0.803 (0.004) & 3.091 (0.017) & 0.458 (0.004) & 0.979 (0.001) \\ 
   & CHDC-ratio & 0.797 (0.004) & 2.261 (0.009) & 0.467 (0.004) & 0.139 (0.002) \\ 
   & CHDC-discount & 0.802 (0.004) & 2.947 (0.017) & 0.463 (0.004) & 0.967 (0.001) \\ 
   \hline
\multirow{7}{*}{Buchanan} & HDC (train) & 0.625 (0.004) & 1.000 (0.000) & 0.625 (0.004) & - \\ 
   & HDC & 0.640 (0.004) & 1.000 (0.000) & 0.640 (0.004) & - \\ 
   & Inv. quantile & 0.799 (0.003) & 1.782 (0.005) & 0.616 (0.004) & 0.599 (0.000) \\ 
   & Penalized & 0.796 (0.003) & 2.901 (0.014) & 0.613 (0.004) & 0.016 (0.001) \\ 
   & Similarity & 0.792 (0.004) & 2.979 (0.015) & 0.585 (0.004) & 0.986 (0.001) \\ 
   & CHDC-ratio & 0.796 (0.003) & 1.643 (0.007) & 0.615 (0.003) & 0.600 (0.004) \\ 
   & CHDC-discount & 0.791 (0.004) & 2.761 (0.016) & 0.596 (0.004) & 0.986 (0.001) \\ 
   \hline
\multirow{7}{*}{Mitt} & HDC (train) & 0.616 (0.003) & 1.000 (0.000) & 0.616 (0.003) & - \\ 
   & HDC & 0.641 (0.003) & 1.000 (0.000) & 0.641 (0.003) & - \\ 
   & Inv. quantile & 0.805 (0.003) & 1.877 (0.005) & 0.616 (0.003) & 0.616 (0.000) \\ 
   & Penalized & 0.794 (0.003) & 3.041 (0.013) & 0.589 (0.003) & 0.066 (0.001) \\ 
   & Similarity & 0.802 (0.003) & 3.113 (0.014) & 0.588 (0.003) & 0.942 (0.001) \\ 
   & CHDC-ratio & 0.798 (0.003) & 1.743 (0.006) & 0.616 (0.003) & 0.636 (0.002) \\ 
   & CHDC-discount & 0.802 (0.003) & 2.996 (0.014) & 0.599 (0.003) & 0.945 (0.001) \\ 
   \hline
\multirow{7}{*}{Stella} & HDC (train) & 0.586 (0.004) & 1.000 (0.000) & 0.586 (0.004) & - \\ 
   & HDC & 0.607 (0.004) & 1.000 (0.000) & 0.607 (0.004) & - \\ 
   & Inv. quantile & 0.809 (0.003) & 2.034 (0.006) & 0.567 (0.004) & 0.608 (0.000) \\ 
   & Penalized & 0.806 (0.003) & 2.914 (0.015) & 0.556 (0.004) & 0.292 (0.003) \\ 
   & Similarity & 0.797 (0.003) & 2.778 (0.015) & 0.581 (0.004) & 0.823 (0.002) \\ 
   & CHDC-ratio & 0.802 (0.003) & 1.707 (0.010) & 0.582 (0.004) & 0.892 (0.002) \\ 
   & CHDC-discount & 0.798 (0.004) & 2.320 (0.014) & 0.580 (0.004) & 0.862 (0.002) \\ 
   \hline
\multirow{7}{*}{Superchris} & HDC (train) & 0.733 (0.003) & 1.000 (0.000) & 0.733 (0.003) & - \\ 
   & HDC & 0.754 (0.003) & 1.000 (0.000) & 0.754 (0.003) & - \\ 
   & Inv. quantile & 0.803 (0.003) & 1.282 (0.006) & 0.730 (0.003) & 0.601 (0.000) \\ 
   & Penalized & 0.797 (0.004) & 2.698 (0.016) & 0.713 (0.003) & 0.120 (0.002) \\ 
   & Similarity & 0.799 (0.004) & 2.683 (0.016) & 0.704 (0.004) & 0.918 (0.002) \\ 
   & CHDC-ratio & 0.795 (0.003) & 1.179 (0.006) & 0.720 (0.003) & 0.539 (0.004) \\ 
   & CHDC-discount & 0.801 (0.004) & 2.188 (0.015) & 0.720 (0.003) & 0.926 (0.002) \\ 
   \hline
\end{tabular}

    \caption{Full results for odor decoding with alpha=$0.2$. Results aggregated over 500 independent repetitions. For all conformal methods, the odor data is split into training ($50\%$), calibration ($40\%$), and testing ($10\%$) sets. The standard HDC classifier utilizes the combined training and calibration folds, whereas HDC (train) uses the training fold only. Point-valued ConformalHDC achieves performance comparable to HDC (train), though it performs slightly lower than the full HDC model. This is due to data splitting constraints on the relatively small neural dataset ($\approx 600$ samples per subject).}
    \label{tab:rat}
\end{table}

\end{document}